\documentclass{article}

\usepackage{microtype}
\usepackage{graphicx}
\usepackage{subfigure}
\usepackage{booktabs} % for professional tables
\usepackage{multirow}
\usepackage{makecell}
\usepackage{comment}

\usepackage{hyperref}

\usepackage{xcolor}
\usepackage{colortbl}
\usepackage{color}
\definecolor{gray}{rgb}{0.5,0.5,0.5}

\usepackage[accepted]{icml2023}

\usepackage{amsmath}
\usepackage{amssymb}
\usepackage{mathtools}
\usepackage{amsthm}
\newcommand{\mname}{LinSAT}

\usepackage[capitalize,noabbrev]{cleveref}

\theoremstyle{plain}
\newtheorem{theorem}{Theorem}[section]

\newtheorem{lemma}[theorem]{Lemma}
\newtheorem{corollary}[theorem]{Corollary}
\theoremstyle{definition}

\theoremstyle{remark}

\usepackage[textsize=tiny]{todonotes}

\newcommand{\yanr}[1]{{#1}}

\DeclareMathOperator*{\argmin}{argmin}
\DeclareMathOperator*{\argmax}{argmax}

\icmltitlerunning{{\mname}Net: The Positive Linear Satisfiability Neural Networks}
\begin{document}

\twocolumn[
%\icmltitle{The SK4PLC-Layer: Encoding Positive Linear Constraints in Neural Networks via Sinkhorn Iterations}
%\icmltitle{{\mname}Net: Enforcing Positive Linear Satisfiability in One-shot \\ for End-to-end Neural Networks}
\icmltitle{{\mname}Net: The Positive Linear Satisfiability Neural Networks}



\begin{icmlauthorlist}
\icmlauthor{Runzhong Wang}{yyy,zzz}
\icmlauthor{Yunhao Zhang}{yyy}
\icmlauthor{Ziao Guo}{yyy}
\icmlauthor{Tianyi Chen}{yyy}
\icmlauthor{Xiaokang Yang}{yyy}
\icmlauthor{Junchi Yan}{yyy,zzz}
\end{icmlauthorlist}

\icmlaffiliation{yyy}{Department of Computer Science and Engineering, and MoE Key Lab of Artificial Intelligence, Shanghai Jiao Tong University}
\icmlaffiliation{zzz}{Shanghai AI Laboratory}

\icmlcorrespondingauthor{Junchi Yan}{yanjunchi@sjtu.edu.cn}

\icmlkeywords{Machine Learning, ICML}

\vskip 0.3in
]



\printAffiliationsAndNotice{}  % leave blank if no need to mention equal contribution
\begin{abstract}
Encoding constraints into neural networks is attractive. This paper studies how to introduce the popular positive linear satisfiability to neural networks. We propose the first differentiable satisfiability layer based on an extension of the classic Sinkhorn algorithm for jointly encoding multiple sets of marginal distributions. We further theoretically characterize the convergence property of the Sinkhorn algorithm for multiple marginals. In contrast to the sequential decision e.g.\ reinforcement learning-based solvers, we showcase our technique in solving constrained (specifically satisfiability) problems by one-shot neural networks, including i) a neural routing solver learned without supervision of optimal solutions; ii) a partial graph matching network handling graphs with unmatchable outliers on both sides; iii) a predictive network for financial portfolios with continuous constraints. To our knowledge, there exists no one-shot neural solver for these scenarios when they are formulated as satisfiability problems. Source code is available at \url{https://github.com/Thinklab-SJTU/LinSATNet}.

\end{abstract}

\section{Introduction}%\vspace{-4pt}
\label{sec:intro}
 
It remains open for how to effectively encode the constraints into neural networks for decision-making beyond unconstrained regression and classification. Roughly speaking, we distinguish two categories of such constrained problems: \emph{optimization} and \emph{decision}. \emph{Optimization} problems consider explicit objective functions that are directly related to downstream tasks, whereby their optimization forms are usually more complicated. \emph{Decision} problems do not consider the objective of the downstream task, or the downstream task may not have any explicit objectives. It is possible that \emph{decision} problems also have underlying forms, however, their objectives are usually interpreted as ``finding a feasible solution nearest to the input''. 
In particular, the \emph{decision} problem can be divided into two cases: i) only judge if there exists a feasible solution or not; ii) output a feasible solution close to an unconstrained input. This paper focuses on the latter case for \emph{decision} problem, and we term it as \emph{satisfiability} problem if not otherwise specified.

\begin{table}[tb!]
    \centering
    \caption{Comparison of constraint-encoding for neural networks in finding a solution, with/without an explicit objective function. Note the other works on enforcing a certain kind of \emph{satisfiability} e.g. \emph{permute}, \emph{rank}, \emph{match}, (but not boolean-SAT) can be incorporated by the positive linear constraints as fulfilled by our {\mname}.}
    \resizebox{\linewidth}{!}{
    \begin{tabular}{r|ccc}
    \toprule
    Paper & Formulation & Constraint type & Exact gradient? \\
    \midrule
    \citet{AmosICML17} & \emph{optim.}  & linear & Yes \\
    \citet{PoganvcicICLR19} & \emph{optim.} & combinatorial & No \\
    \citet{BerthetNIPS20} & \emph{optim.} & combinatorial & No \\
    \citet{WangICML19} & \emph{optim.} & combinatorial & Yes \\
    \midrule
     \citet{NeuroSAT19} & \emph{sat.}  & boolean-SAT& Yes \\
    \citet{CruzCVPR17} & \emph{sat.} & permutation & Yes \\
    \citet{AdamsArxiv11} & \emph{sat.}  & ranking & Yes \\
    \citet{WangICCV19} & \emph{sat.}  & matching& Yes \\
    {\mname} (ours) & \emph{sat.} & positive linear & Yes \\
    \bottomrule
    \end{tabular}
    }
    \vspace{-10pt}
    \label{tab:compare}
\end{table}

\begin{figure*}[t!]
    \centering
    \includegraphics[width=0.9\textwidth]{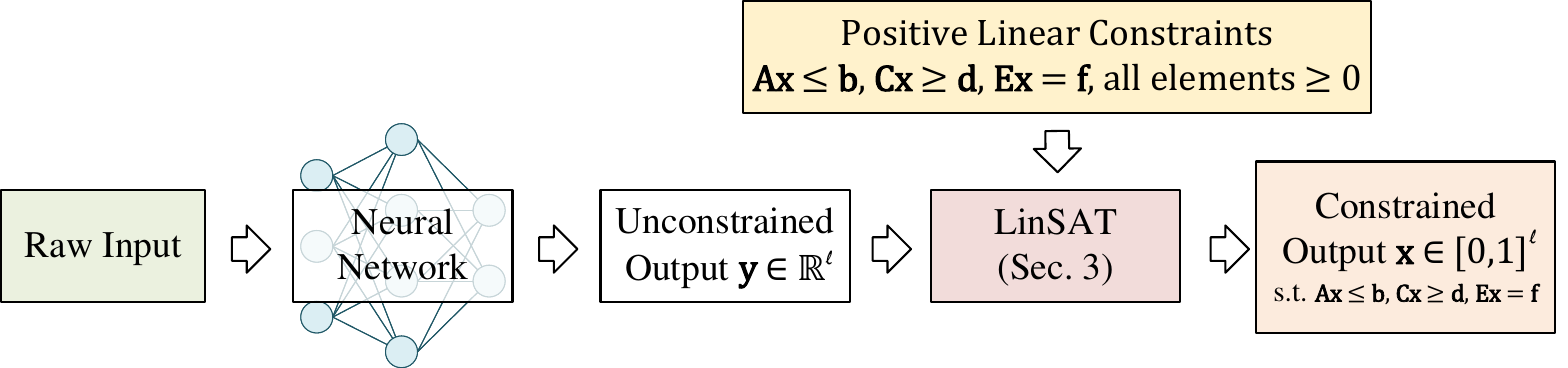}
    \caption{The use case of {\mname}Net. If the last layer of the neural network is a linear layer whose output is unconstrained, {\mname} enforces the \emph{satisfiability} of positive linear constraints to the final output.}
    \vspace{-10pt}
    \label{fig:usecase}
\end{figure*}

Notably, machine learning has been well adopted in solving both optimization and decision problems, especially for combinatorial optimization (CO)~\cite{bengio2020machine} and SAT problem~\cite{GuoSATMIR23,LiKDD23}. It is relatively easy to introduce learning into problem-solving as a building block under the traditional solving framework~\cite{WangCVPR21,WangNIPS21}, yet it is more attractive to develop a learning-based framework in a more systematic manner. In this regard, reinforcement learning (RL)~\cite{LiuICLR23} or alternative sequence-to-sequence models~\cite{VinyalsNIPS15} that solve the problem in an auto-regressive way is of prominence adoption, while they are often less efficient for their sequential decision nature. Thus efforts have also been put into one-shot problem solving, and a popular alternative is designing certain penalties in the loss~\citep{KaraliasNIPS20} to respect the constraints. Being more thought-provoking, a more aggressive ambition is to develop end-to-end differentiable neural networks whereby the constraints are seamlessly encoded in their architecture, such that the efficiency of neural networks for one-shot solving can be fulfilled.

As shown in \cref{tab:compare}, a number of special constraints have been successfully encoded in neural networks by designing a certain layer for end-to-end training and solving, and as will be shown later, our proposed {\mname}Net (i.e.\ Positive Linear Satisfiability Network) can incorporate the other studied constraints for satisfiability problem and our method is end-to-end differentiable with exact gradient computing. As a side note of \cref{tab:compare}, there exist methods for general combinatorial optimization, e.g.\ \citet{PoganvcicICLR19,PaulusICML21} over inner-product objectives, while it is often at the cost of being inaccessible to the exact gradient in model training which can potentially hurt the performance.

Our work is technically inspired by the recent success of enforcing the satisfiability of certain constraints to neural networks by deliberately handcrafted layers and in particular the classic Sinkhorn algorithm~\citep{sinkhorn1967concerning}. The Sinkhorn algorithm has been a popular technique in recent constrained neural networks: 1)~the permutation constraint to solve jigsaw puzzles by \citet{CruzCVPR17}; 2)~the ranking constraint by \citet{AdamsArxiv11,CuturiNIPS19} and the cardinality (top-$k$) constraint by \citet{XieNIPS20,WangICLR23}; 3)~the one-to-one correspondence constraint for graph matching (with outliers at most in one graph) by \citet{WangPAMI20,YuICLR20,WangPAMI22}. In this paper, we extend the scope of Sinkhorn algorithm to enforce the satisfiability of the more general positive linear constraints~\citep{LubySTOC93}. Denote $\mathbf{x}\in[0,1]^l$ as the $l$-dim output, our proposed {\mname} layer jointly enforces the following constraints:
\begin{equation}
    \mathbf{A}\mathbf{x}\leq \mathbf{b}, \mathbf{C}\mathbf{x}\geq \mathbf{d}, \mathbf{E}\mathbf{x}= \mathbf{f}, \mathbf{x}\in[0,1]^l
    \label{eq:positive-linear-constraint}
\end{equation}
where all elements in $\mathbf{A}, \mathbf{b}, \mathbf{C}, \mathbf{d}, \mathbf{E}, \mathbf{f}$ are non-negative. In fact, this family of constraints incorporates a large scope of real-world scenarios such as the ``packing'' constraints ($\mathbf{A}\mathbf{x}\leq \mathbf{b}$) and the ``covering'' constraints ($\mathbf{C}\mathbf{x}\geq \mathbf{d}$), whereby the aforementioned constraints 1)-3) could be viewed as subsets of Eq.~(\ref{eq:positive-linear-constraint}). The use case of such a positive linear constrained network is illustrated in \cref{fig:usecase}.

In this paper, we first generalize Sinkhorn algorithm to handle multiple sets of marginal distributions. Our multi-set version follows the classic single-set algorithm that is non-parametric (i.e.\ without trainable parameters for a neural network) and involves only matrix-vector arithmetic operations for exact gradient computing and back-propagation. We theoretically characterize the convergence guarantee and its rate of the proposed algorithm concerning the KL divergence and the $L_1$ distance to the target marginal distributions. We further show that any positive linear constraints could be equivalently written as multiple sets of marginal distributions, such that the satisfiability of positive linear constraints could be enforced to a differentiable network. \textbf{The contributions of the paper are:}

1) We generalize the Sinkhorn algorithm to handle multiple sets of marginal distributions, with the theoretical guarantee that the proposed multi-set algorithm preserves the convergence of the single-set version. Our multi-set algorithm offers theoretical and technical groundings for handling the general positive linear constraints. It may also be of independent interest to the area of matrix normalization.

2) We design {\mname}, a differentiable yet parameter-free light-weighted layer to encode the positive linear constraints, based on our devised multi-set Sinkhorn algorithm. The satisfiability layer only involves matrix-vector arithmetic operations and can be strictly enforced, \yanr{decoupling the constraint satisfaction from the learning objective}. To our best knowledge, this is the first differentiable satisfiability layer addressing the general positive linear constraints.

3) To demonstrate its wide applicability, we deploy our {\mname} to three scenarios regarding constrained routing, outlier-aware matching, and predictive portfolio allocation. In these cases, an explicit objective function is difficult to define and a satisfiable solution is the purpose.

\section{Related Work}%\vspace{-4pt}
\label{sec:related}

\textbf{Sinkhorn Algorithm}  \citep{sinkhorn1967concerning} projects a positive matrix to a doubly-stochastic matrix by alternatively normalizing its rows and columns, and \citet{CuturiNIPS13} identifies the connection of Sinkhorn and optimal transport. Its effectiveness also motivates recent theoretical studies concerning its convergence rate~\citep{AltschulerNIPS17,knight2008sinkhorn}, whereby \citet{chakrabarty2021better} offers a comprehensive theory. Due to its differentiability, Sinkhorn is widely applied in vision~\citep{CruzCVPR17,WangICCV19} and learning~\citep{AdamsArxiv11,XieNIPS20} by enforcing specific constraints. However, as summarized in \cref{tab:compare}, the types of constraints studied in the previous works are less general than the positive linear constraints studied in this paper. Our paper also differs from the existing study of multi-marginal optimal transport \citep{pass2015multi} since their ``multi-marginal'' means moving one source distribution to multiple targets, while we are moving multiple sources to multiple targets. To distinguish, we name our algorithm as ``Sinkhorn for multi-\emph{set} marginals''.

\textbf{Approximate Solvers for Positive Linear Programming} 
is also an active topic in theoretical computer science. Despite their solid theoretical groundings, this line of works may not be readily integrated into neural networks. For example, \citet{awerbuch2008stateless,allen2014using} are non-differentiable because their algorithms involve max operations and thresholding functions, respectively. \citet{young2001sequential,LubySTOC93} are neither differentiable due to their incremental steps. Another drawback of these methods is that most of them cannot handle a mix of packing ($\mathbf{A}\mathbf{x}\leq \mathbf{b}$) and covering ($\mathbf{C}\mathbf{x}\geq \mathbf{d}$) constraints except for \citet{young2001sequential}. In this paper, we emphasize differentiability to make it compatible with neural networks, and our method could handle any combinations of packing, covering, and equality constraints. 

\textbf{Differentiable Solvers for Constrained Optimization} address the problem with objective functions and constraints whereby deep graph matching~\cite{YanIJCAI20,YuICLR20} has been a prominent topic with a quadratic objective. \citet{AmosICML17} shows the differentiability at optimal solutions via KKT conditions and presents a case study for quadratic programming. \citet{WangICML19} approximately solves MAXSAT by a differentiable semi-definitive solver. Another line of works develops approximate gradient wrappers for combinatorial solvers: \citet{PoganvcicICLR19} estimates the gradient by the difference of two forward passes; \citet{BerthetNIPS20} estimates the gradient via a batch of random perturbations.

Our approach is devoted to the \emph{satisfiability} setting whereby no explicit objective function is given for the downstream task \citep{NeuroSAT19}. Note that this is more than just a mathematical assumption: in reality, many problems cannot be  defined with an explicit objective function, either due to e.g.\ the missing of some key variables in noisy or dynamic environments, especially when the objective concerns with a future outcome as will be shown in case studies on partial graph matching (Sec.~\ref{sec:gm}) and predictive portfolio allocation (Sec.~\ref{sec:portfolio}). However, existing neural networks for \emph{optimization} (e.g.\ \citet{butler2021integrating} for asset allocation) do not adapt smoothly to these realistic scenarios.

Finally, note that the boolean satisfiability problem~\cite{CookSTOC71} also receives attention from machine learning community~\cite{GuoSATMIR23}, whereby end-to-end neural nets have also been actively developed e.g.\ NeuroSAT~\cite{NeuroSAT19} and QuerySAT~\cite{QuerySAT21}. As we mentioned in~\cref{tab:compare}, the boolean-SAT cannot be covered by our constraint and is orthogonal to this work.

\begin{algorithm}[tb!]
   \caption{Sinkhorn for Single-Set Marginals (Classic)}
   \label{alg:sk}
\begin{algorithmic}[1]
   \STATE {\bfseries Input:} score matrix $\mathbf{S}\in\mathbb{R}_{\geq 0}^{m\times n}$, single set of marginal distributions $\mathbf{v}\in\mathbb{R}_{\geq 0}^{m}, \mathbf{u}\in\mathbb{R}_{\geq 0}^{n}$.
   \STATE Initialize $\Gamma_{i,j}=\frac{s_{i,j}}{\sum_{i=1}^m s_{i,j}}$;
   \REPEAT
   \STATE ${\Gamma}_{i,j}^{\prime} = \frac{{\Gamma}_{i,j}v_{i}}{\sum_{j=1}^n {\Gamma}_{i,j}u_{j}}$; $\triangleright$ normalize w.r.t.\ $\mathbf{v}$
   \STATE ${\Gamma}_{i,j} = \frac{{\Gamma}_{i,j}^{\prime}u_j}{\sum_{i=1}^m {\Gamma}_{i,j}^{\prime}u_j}$; $\triangleright$ normalize w.r.t.\ $\mathbf{u}$
   \UNTIL{convergence}
\end{algorithmic}
\end{algorithm}

\section{Methodology}%\vspace{-4pt}
\label{sec:method}
% two-pages
Sec.~\ref{sec:method-preliminary-sinkhorn} formulates  the classic Sinkhorn algorithm handling a single set of marginal distributions. Sec.~\ref{sec:extended-sk} proposes the generalized multi-set Sinkhorn with a convergence study. In Sec.~\ref{sec:linsat} we devise {\mname} layer to enforce the positive linear constraints, by connecting to the marginal distributions.

\subsection{Preliminaries: The Classic Sinkhorn Algorithm for Single Set of Marginal Distributions}
\label{sec:method-preliminary-sinkhorn}
%\vspace{-4pt}
We first revisit the classic Sinkhorn algorithm in \cref{alg:sk}, which is a differentiable method developed by \citet{sinkhorn1967concerning} to enforce a single set of marginal distributions to a matrix. 
Given non-negative score matrix $\mathbf{S}\in\mathbb{R}_{\geq 0}^{m\times n}$ and a set of marginal distributions on rows $\mathbf{v}\in \mathbb{R}_{\geq 0}^m$ and columns $\mathbf{u} \in \mathbb{R}_{\geq 0}^n$ (where $\sum_{i=1}^m v_i = \sum_{j=1}^n u_j = h$), the Sinkhorn algorithm outputs a normalized matrix $\mathbf{\Gamma}\in[0,1]^{m\times n}$ so that $\sum_{i=1}^m \Gamma_{i,j}u_{j}=u_j, \sum_{j=1}^n \Gamma_{i,j}u_{j}=v_i$. Conceptually, $\Gamma_{i,j}$ means the proportion of $u_j$ moved to $v_i$. Note that $\Gamma_{i,j}$ usually has no same meaning in the ``reversed move'' from $v_i$ to $u_j$ if $v_i\neq u_j$\footnote{This formulation is modified from the conventional formulation where $\Gamma_{i,j}u_j$ is equivalent to the elements in the ``transport'' matrix in \citet{CuturiNIPS13}. We prefer this formulation as it seamlessly generalizes to multi-set marginals. See Appendix~\ref{sec:discuss_with_classic_sinkhorn} for details.}.
We initialize $\mathbf{\Gamma}^{(0)}$ by 
\begin{equation}
    \Gamma_{i,j}^{(0)}=\frac{s_{i,j}}{\sum_{i=1}^m s_{i,j}}.
\end{equation}
At iteration $t$, $\mathbf{\Gamma}^{\prime(t)}$ is obtained by normalizing w.r.t.\ the row-distributions $\mathbf{v}$, and $\mathbf{\Gamma}^{(t+1)}$ is obtained by normalizing w.r.t\ the column-distributions $\mathbf{u}$. $\mathbf{\Gamma}^{(t)}, \mathbf{\Gamma}^{\prime(t)} \in [0,1]^{m\times n}$ are scaled by $\mathbf{u}$ before and after normalization. Specifically,
\begin{equation}
    {\Gamma}_{i,j}^{\prime(t)} = \frac{{\Gamma}_{i,j}^{(t)}v_i}{\sum_{j=1}^n {\Gamma}_{i,j}^{(t)}u_j},\quad  {\Gamma}_{i,j}^{(t+1)} = \frac{{\Gamma}_{i,j}^{\prime(t)}u_{j}}{\sum_{i=1}^m {\Gamma}_{i,j}^{\prime(t)}u_{j}}.
    \label{eq:sk}
\end{equation}
The above algorithm is easy to implement and GPU-friendly. Besides, it only involves matrix-vector arithmetic operations, meaning that it is naturally differentiable and the backward pass is easy to implement by the autograd feature of modern deep learning frameworks~\citep{Paszke2017AutomaticPytorch}.

Some recent theoretical studies \citep{AltschulerNIPS17,chakrabarty2021better} further characterize the rate of convergence of Sinkhorn algorithm. Define the $L_1$ error as the violation of the marginal distributions,
\begin{subequations}
\begin{align}
    L_{1}(\mathbf{\Gamma}^{(t)})&=\| \mathbf{v}^{(t)} -\mathbf{v}\|_1, \ v^{(t)}_i=\sum_{j=1}^n {\Gamma}^{(t)}_{i,j}{u}_j, \\
    L_{1}(\mathbf{\Gamma}^{\prime(t)})&=\|\mathbf{u}^{(t)}-\mathbf{u}\|_1,\  u^{(t)}_j=\sum_{i=1}^m {\Gamma}^{\prime(t)}_{i,j}{u}_j .
\end{align}
\end{subequations}

\begin{theorem}[See \citet{chakrabarty2021better} for the proof]
\label{thm:sk}
For any $\epsilon>0$, the Sinkhorn algorithm for single-set marginals returns a matrix $\mathbf{\Gamma}^{(t)}$ or $\mathbf{\Gamma}^{\prime(t)}$ with $L_1$ error $\leq \epsilon$ in time $t=\mathcal{O}\left(\frac{h^2\log(\Delta/\alpha)}{\epsilon^2}\right)$, where $\alpha={\min_{i,j:s_{i,j}>0} s_{i,j}}/{\max_{i,j} s_{i,j}}$, $\Delta=\max_j \left|\{i: s_{i,j} > 0\}\right|$ is the max number of non-zeros in any column of $\mathbf{S}$, and recall that $\sum_{i=1}^m v_i = \sum_{j=1}^n u_j = h$.
\end{theorem}

\subsection{Generalizing Sinkhorn Algorithm for Multiple Sets of Marginal Distributions}
\label{sec:extended-sk}
%\vspace{-4pt}
Existing literature about the Sinkhorn algorithm mainly focuses on a single set of marginal distributions. In the following, we present our approach that extends the Sinkhorn algorithm into multiple sets of marginal distributions.

Following \citet{CuturiNIPS13}, we view the Sinkhorn algorithm as ``moving masses'' between marginal distributions: $\Gamma_{i,j}\in[0,1]$ means the proportion of $u_i$ moved to $v_j$. Interestingly, it yields the same formulation if we simply replace $\mathbf{u},\mathbf{v}$ by another set of marginal distributions, suggesting the potential of extending the Sinkhorn algorithm to multiple sets of marginal distributions. To this end, we devise \cref{alg:extended-sk}, an extended version of the Sinkhorn algorithm, whereby $k$ sets of marginal distributions are jointly enforced to fit more complicated real-world scenarios. The sets of marginal distributions are $\mathbf{u}_\eta\in \mathbb{R}_{\geq 0}^n, \mathbf{v}_\eta\in \mathbb{R}_{\geq 0}^m$, and we have:
\begin{equation}
    \forall \eta\in \{1, \cdots,k\}: \sum_{i=1}^m v_{\eta,i}=\sum_{j=1}^n u_{\eta,j}=h_\eta.
\end{equation}
It assumes the existence of a normalized $\mathbf{Z} \in [0,1]^{m\times n}$ s.t.
\begin{equation}
   \small{ \forall \eta\in \{1,\cdots, k\}: \sum_{i=1}^m z_{i,j} u_{\eta,j}=u_{\eta,j}, \sum_{j=1}^n z_{i,j} u_{\eta,j}=v_{\eta,i},}
    \label{eq:define-z}
\end{equation}
i.e., the multiple sets of marginal distributions have a non-empty feasible region (see \cref{sec:feasibility-assump} for details). Multiple sets of marginal distributions could be jointly enforced by traversing the Sinkhorn iterations over $k$ sets of marginal distributions. We extend Eq.~(\ref{eq:sk}) for multiple marginals,
\begin{equation}
    {\Gamma}_{i,j}^{\prime(t)} = \frac{{\Gamma}_{i,j}^{(t)}v_{\eta,i}}{\sum_{j=1}^n {\Gamma}_{i,j}^{(t)}u_{\eta,j}},\  {\Gamma}_{i,j}^{(t+1)} = \frac{{\Gamma}_{i,j}^{\prime(t)}u_{\eta,j}}{\sum_{i=1}^m {\Gamma}_{i,j}^{\prime(t)}u_{\eta,j}},
\end{equation}
where $\eta=(t\mod k)+1$ is the index of marginal sets. Similarly to \cref{alg:sk}, this generalized Sinkhorn algorithm finds a normalized matrix that is close to $\mathbf{S}$.

\begin{algorithm}[tb]
   \caption{Sinkhorn for Multi-Set Marginals (Proposed)}
   \label{alg:extended-sk}
\begin{algorithmic}[1]
   \STATE {\bfseries Input:} score matrix $\mathbf{S}\in\mathbb{R}_{\geq 0}^{m\times n}$, $k$ sets of marginal distributions $\mathbf{V}\in\mathbb{R}_{\geq 0}^{k\times m},\mathbf{U}\in\mathbb{R}_{\geq 0}^{k\times n}$.
   \STATE Initialize $\Gamma_{i,j}=\frac{s_{i,j}}{\sum_{i=1}^m s_{i,j}}$;
   \REPEAT
   \FOR{$\eta=1$ {\bfseries to} $k$}
   \STATE ${\Gamma}_{i,j}^{\prime} = \frac{{\Gamma}_{i,j}v_{\eta,i}}{\sum_{j=1}^n {\Gamma}_{i,j}u_{\eta,j}}$; $\triangleright$ normalize w.r.t.\ $\mathbf{v}_\eta$
   \STATE ${\Gamma}_{i,j} = \frac{{\Gamma}_{i,j}^{\prime}u_{\eta,j}}{\sum_{i=1}^m {\Gamma}_{i,j}^{\prime}u_{\eta,j}}$; $\triangleright$ normalize w.r.t.\ $\mathbf{u}_\eta$
   \ENDFOR
   \UNTIL{convergence}
\end{algorithmic}
\end{algorithm}

\textbf{Theoretical Characterization of the Convergence of Multi-set Sinkhorn}.
In the following, we show that our proposed \cref{alg:extended-sk} shares a similar convergence pattern with \cref{alg:sk} and \cref{thm:sk}. We generalize the theoretical steps in \citet{chakrabarty2021better} as follows.

We first study the convergence property of \cref{alg:extended-sk} in terms of Kullback-Leibler (KL) divergence. In the following, we have $\eta = (t\mod k)+1$ unless otherwise specified. We define the probability over marginals $\pi_{v_{\eta,i}} = v_{\eta,i} / h_\eta$, and similarly for $\pi_{u_{\eta,j}}$. $\mathbf{v}^{(t)}_\eta, \mathbf{u}^{(t)}_\eta$ are the $\eta$-th marginal distributions achieved by $\mathbf{\Gamma}^{(t)}$ and $\mathbf{\Gamma}^{\prime(t)}$, respectively,
\begin{equation}
    {v}^{(t)}_{\eta,i} = \sum_{j=1}^n {\Gamma}_{i,j}^{(t)}u_{\eta,j}, \ {u}^{(t)}_{\eta,j} = \sum_{i=1}^m {\Gamma}_{i,j}^{\prime(t)}u_{\eta,j}.
\end{equation}
$D_{\text{KL}}(\pi_{\mathbf{v}_{\eta}}||\pi_{\mathbf{v}_{\eta}^{(t)}})=\sum_{i=1}^m\frac{{v}_{\eta,i}}{h_\eta} \cdot \log \frac{{{v}_{\eta,i}}/{h_\eta}}{{{v}^{(t)}_{\eta,i}}/{h_\eta}}$ denotes the KL divergence between the current marginal achieved by $\mathbf{\Gamma}^{(t)}$ and the target marginal distribution. Similarly, we define $D_{\text{KL}}(\pi_{\mathbf{u}_{\eta}}||\pi_{\mathbf{u}_{\eta}^{(t)}})=\sum_{j=1}^n\frac{{u}_{\eta,j}}{h_\eta} \cdot \log \frac{{{u}_{\eta,j}}/{h_\eta}}{{{u}^{(t)}_{\eta,j}}/{h_\eta}}$ for $\mathbf{\Gamma}^{\prime(t)}$.

\begin{theorem}[Converge Rate for KL divergence]
\label{thm:extended-sk}
For any $\delta>0$, the Sinkhorn algorithm for multi-set marginals returns a matrix $\mathbf{\Gamma}^{(t)}$ or $\mathbf{\Gamma}^{\prime(t)}$ with KL divergence $\leq\delta$ in time $t=\mathcal{O}\left(\frac{k\log(\Delta/\alpha)}{\delta}\right)$, where $\alpha={\min_{i,j:s_{i,j}>0} s_{i,j}}/{\max_{i,j} s_{i,j}}$, $\Delta=\max_j \left|\{i: s_{i,j} > 0\}\right|$ is the max number of non-zeros in any column of $\mathbf{S}$, and recall that $k$ is the number of marginal sets.
\end{theorem}

\begin{proof}
(Sketch only and see Appendix \ref{sec:detailed_proof} for details) To prove the upper bound of the convergence rate, we define the KL divergence for matrices $\mathbf{Z},\mathbf{\Gamma}$ at $\eta$,
\begin{equation}
    D(\mathbf{Z}, \mathbf{\Gamma},\eta)=\frac{1}{h_\eta}\sum_{i=1}^m\sum_{j=1}^n z_{i,j}u_{\eta,j} \log \frac{z_{i,j}}{\Gamma_{i,j}},
\end{equation}
We then prove the convergence rate w.r.t.\ KL divergence based on the following two Lemmas.
\begin{lemma}
\label{lemma:kl_gamma0}
For any $\eta$, $D(\mathbf{Z}, \mathbf{\Gamma}^{(0)}, \eta) \leq \log(1+2\Delta/\alpha)$.
\end{lemma}
\begin{lemma}
\label{lemma:kl_diff}
For $\eta = (t\mod k)+1$, $\eta^\prime = (t+1\mod k)+1$, we have
\begin{align}
    D(\mathbf{Z}, \mathbf{\Gamma}^{(t)},\eta) - D(\mathbf{Z},\mathbf{\Gamma}^{\prime(t)},\eta) &= D_{\text{KL}}(\pi_{\mathbf{v}_\eta}|| \pi_{\mathbf{v}^{(t)}_\eta}) \notag\\
    D(\mathbf{Z}, \mathbf{\Gamma}^{\prime(t)},\eta) - D(\mathbf{Z}, \mathbf{\Gamma}^{(t+1)},\eta^\prime) &= D_{\text{KL}}(\pi_{\mathbf{u}_\eta}||\pi_{\mathbf{u}^{(t)}_\eta}) \notag
\end{align}
\end{lemma}
The proof of these two Lemmas is referred to the appendix. Denote $T=\frac{k\log(1+2\Delta/\alpha)}{\delta}+\zeta$, where $\zeta \in [0, k)$ is a residual term ensuring $(T+1)\mod k = 0$. If all $D_{\text{KL}}(\pi_{\mathbf{v}_\eta}|| \pi_{\mathbf{v}^{(t)}_\eta}) > \delta$ and $D_{\text{KL}}(\pi_{\mathbf{u}_\eta}|| \pi_{\mathbf{u}^{(t)}_\eta}) > \delta$ for all $\eta$, by substituting \cref{lemma:kl_diff} and summing, we have
\begin{equation}
    D(\mathbf{Z}, \mathbf{\Gamma}^{(0)},1) -  D(\mathbf{Z}, \mathbf{\Gamma}^{(T+1)},1) > \frac{T\delta}{k} \geq \log(1+2\Delta/\alpha).\notag
\end{equation}
As KL divergence is non-negative, the above formula contradicts to \cref{lemma:kl_gamma0}. It ends the proof of \cref{thm:extended-sk}.
\end{proof}

For the $L_1$ error defined as
\begin{align}
    L_{1}(\mathbf{\Gamma}^{(t)})=\|\mathbf{v}_\eta^{(t)}-\mathbf{v}_\eta\|_1,\ 
    L_{1}(\mathbf{\Gamma}^{\prime(t)})=\|\mathbf{u}^{(t)}_\eta-\mathbf{u}_\eta\|_1, \notag
\end{align}
we have the following corollary for \cref{alg:extended-sk},
\begin{corollary}[Converge Rate for $L_1$-error]
\label{coro:extended-sk-L1}
For any $\epsilon>0$, the Sinkhorn algorithm for multi-set marginals returns a matrix $\mathbf{\Gamma}^{(t)}$ or $\mathbf{\Gamma}^{\prime(t)}$ with $L_1$ error $\leq\epsilon$ in time $t=\mathcal{O}\left(\frac{\hat{h}^2k\log(\Delta/\alpha)}{\epsilon^2}\right)$ where $\hat{h}= \max_\eta \sum_{i=1}^m v_{\eta,i}$.
\end{corollary}
\begin{proof}
Without loss of generality, we apply Pinsker's inequality $D_{\text{KL}}(\mathbf{p}||\mathbf{q})\geq \frac{1}{2}\|\mathbf{p}-\mathbf{q}\|_1^2$ to the row marginal distributions, and we have:
\begin{subequations}
\begin{align}
    D_{\text{KL}}(\pi_{\mathbf{v}_\eta}|| \pi_{\mathbf{v}^{(t)}_\eta}) &\geq \frac{1}{2h^2_\eta}\|\mathbf{v}^{(t)}_\eta-\mathbf{v}_\eta\|_1^2\\
    &\geq \frac{1}{2\hat{h}^2}\|\mathbf{v}^{(t)}_\eta-\mathbf{v}_\eta\|_1^2.
\end{align}
\end{subequations}
By setting $\delta=\frac{\epsilon^2}{2\hat{h}^2}$ in \cref{thm:extended-sk}, in $t=\mathcal{O}\left(\frac{\hat{h}^2k\log(\Delta/\alpha)}{\epsilon^2}\right)$ we have $\delta \geq D_{\text{KL}}(\pi_{\mathbf{v}_\eta}||\pi_{\mathbf{v}^{(t)}_\eta})$, corresponding to $\|\mathbf{v}^{(t)}_\eta-\mathbf{v}_\eta\|_1 \leq \epsilon$. Same conclusion holds for $\|\mathbf{u}^{(t)}_\eta-\mathbf{u}_\eta\|_1 \leq \epsilon$.
\end{proof}

\textbf{Discussion of the Underlying Formulation of Multi-set Sinkhorn}. We empirically discover that \cref{alg:extended-sk} finds a $\mathbf{\Gamma}$ that is close to the input $\mathbf{S}$. For the most general case, the exact underlying formulation is unknown and requires future efforts. We are able to characterize the formulation under a special (but general enough) case, discussed as follows:

When all $u_{\eta, j}$s are binary, i.e.\ $\forall i, j: u_{\eta, j}=0$ or $c$ (a constant), by introducing $\mathbf{W}\in\mathbb{R}^{m\times n}$ and $\mathbf{W}=\tau\log\mathbf{S}$, Sinkhorn for multi-set marginals tackles the following entropic regularized linear problem:
\begin{subequations}
\label{eq:extended-sk-obj}
\begin{align}
    & \min_\mathbf{\Gamma} \ -tr(\mathbf{W}^\top\mathbf{\Gamma}) + \tau\sum_{i,j}\Gamma_{i,j}\log\Gamma_{i,j}, \\
     s.t. \quad &\mathbf{\Gamma}\in[0,1]^{m\times n}, \quad \forall \eta \in \{1,\cdots,k\}:\\ 
      & \quad \sum_{i=1}^m \Gamma_{i,j}u_{\eta,j}=u_{\eta,j}, \sum_{j=1}^n\Gamma_{i,j}u_{\eta,j}=v_{\eta,i}.
\end{align}
\end{subequations}
where $\tau$ is the temperature hyperparameter for entropic regularization. Eq.~(\ref{eq:extended-sk-obj}) is tackled by first applying $\mathbf{S}=\exp \left( \mathbf{W} / \tau \right)$, and then calling \cref{alg:extended-sk}. When $\tau\rightarrow0^+$, Eq.~(\ref{eq:extended-sk-obj}) degenerates to a linear programming problem. Since all constraints are positive linear and the objective is also linear, the solution to Eq.~(\ref{eq:extended-sk-obj}) usually lies at the vertices of the feasible space when $\tau\rightarrow0^+$ i.e.\ most elements in $\mathbf{\Gamma}$ will be close to 0/1 given a small $\tau$. Such a property is also in line with the ``classic'' single-set Sinkhorn (with entropic regularization), where $\tau$ controls the ``discreteness'' of $\mathbf{\Gamma}$. Please refer to \cref{sec:discuss_with_tau} for more details.

\subsection{{\mname}: Enforcing Positive Linear Satisfiability}
\label{sec:linsat}
%\vspace{-4pt}
Denote $\mathbf{y}$ as an $l$-length vector that can be the output of any neural network. Our {\mname} develops an satisfiability layer that projects $\mathbf{y}$ into $\mathbf{x}\in[0,1]^l$, $\text{\mname}(\mathbf{y}, \mathbf{A}, \mathbf{b}, \mathbf{C}, \mathbf{d}, \mathbf{E}, \mathbf{f}) \rightarrow \mathbf{x}$, where $\mathbf{A}\mathbf{x}\leq \mathbf{b}, \mathbf{C}\mathbf{x}\geq \mathbf{d}, \mathbf{E}\mathbf{x}= \mathbf{f}$. $\mathbf{x}$ is dependent on $\mathbf{y}$ (following Eq.~(\ref{eq:extended-sk-obj})) and, in the meantime, lies in the feasible space. We firstly show how to encode $\mathbf{y}$ and $\mathbf{x}$ by our proposed \cref{alg:extended-sk}.

\textbf{Encoding Neural Network's Output}.
For an $l$-length vector denoted as $\mathbf{y}$, the following matrix is built
\begin{equation}
    \mathbf{W} = \left[
    \begin{array}{cccc|c}
         {y}_1 & {y}_2 &...& {y}_l & \beta  \\
         \hline
         \beta  & \beta  &... & \beta  & \beta
    \end{array}
    \right],
\end{equation}
where $\mathbf{W}$ is of size $2 \times (l+1)$, and $\beta$ is the dummy variable e.g.\ $\beta=0$. $\mathbf{y}$ is put at the upper-left region of $\mathbf{W}$. The entropic regularizer is then enforced to control discreteness and handle potential negative inputs:
\begin{equation}
    \mathbf{S} = \exp \left(\frac{\mathbf{W}}{\tau}\right).
\end{equation}
The score matrix $\mathbf{S}$ is taken as the input of \cref{alg:extended-sk}. {\mname} then enforces positive linear constraints to the corresponding region of $\mathbf{y}$ by regarding the constraints as marginal distributions. 

\textbf{From Linear Constraints to Marginal Distributions}. 
We discuss the connections between positive linear constraints and marginal distributions for $\mathbf{A}\mathbf{x}\leq \mathbf{b}, \mathbf{C}\mathbf{x}\geq \mathbf{d}, \mathbf{E}\mathbf{x}= \mathbf{f}$, respectively. For notation's simplicity, here we discuss with only one constraint. Multiple constraints are jointly enforced by multiple sets of marginals.

\textit{Packing constraint} $\mathbf{A}\mathbf{x}\leq \mathbf{b}$. Assuming that there is only one constraint, we rewrite the constraint as $\sum_{i=1}^l a_ix_i \leq b$. The marginal distributions are defined as 
\begin{equation}
    \mathbf{u}_p = \underbrace{\left[a_1 \quad a_2 \quad ...\quad a_l \quad b\right]}_{l \text{ dims}+1 \text{ dummy dim}}, \ 
    \mathbf{v}_p = \left.\left[
    \begin{array}{c}
         b  \\
         \sum_{i=1}^l a_i 
    \end{array}
    \right]\right..
\end{equation}
Following the ``transportation'' view of Sinkhorn~\citep{CuturiNIPS13}, the output $\mathbf{x}$ \emph{moves} at most $b$ unit of mass from $a_1, a_2, \cdots, a_l$, and the dummy dimension allows the inequality by \emph{moving} mass from the dummy dimension. It is also ensured that the sum of $\mathbf{u}_p$ equals the sum of $\mathbf{v}_p$. 

\textit{Covering constraint} $\mathbf{C}\mathbf{x}\geq \mathbf{d}$. Assuming that there is only one constraint, we rewrite the constraint as $\sum_{i=1}^l c_ix_i\geq d$. The marginal distributions are defined as 
\begin{equation}
    \mathbf{u}_c = \underbrace{\left[c_1 \quad c_2 \quad ...\quad c_l \quad \gamma d\right]}_{l \text{ dims} + 1 \text{ dummy dim}}, \ 
    \mathbf{v}_c = \left.\left[
    \begin{array}{c}
         (\gamma+1) d  \\
         \sum_{i=1}^l c_i - d 
    \end{array}
    \right]\right.,
\end{equation}
where the multiplier $\gamma=\left\lfloor\sum_{i=1}^lc_i / d \right\rfloor$ is necessary because we always have $\sum_{i=1}^l c_i \geq d$ (else the constraint is infeasible), and we cannot reach the feasible solution where all elements in $\mathbf{x}$ are 1s without this multiplier. This formulation ensures that at least $d$ unit of mass is \emph{moved} from $c_1, c_2, \cdots, c_l$ by $\mathbf{x}$, thus representing the covering constraint of ``greater than''.  It is also ensured that the sum of $\mathbf{u}_c$ equals the sum of $\mathbf{v}_c$.

\textit{Equality constraint} $\mathbf{E}\mathbf{x}= \mathbf{f}$. Representing the equality constraint is more straightforward. Assuming that there is only one constraint, we rewrite the constraint as $\sum_{i=1}^l e_ix_i= f$. The marginal distributions are defined as 
\begin{equation}
    \mathbf{u}_e = \underbrace{\left[e_1 \quad e_2 \quad ...\quad e_l \quad 0\right]}_{l \text{ dims} + \text{dummy dim}=0}, \ 
    \mathbf{v}_e = \left.\left[
    \begin{array}{c}
         f  \\
         \sum_{i=1}^l e_i - f 
    \end{array}
    \right]\right.,
\end{equation}
where the output $\mathbf{x}$ \emph{moves} $e_1, e_2, \cdots, e_l$ to $f$, and we need no dummy element in $\mathbf{u}_e$ because it is an equality constraint. It is also ensured that the sum of $\mathbf{u}_e$ equals the sum of $\mathbf{v}_e$.

\textbf{Enforcing Multiple Constraints by Sinkhorn}. 
The constraints are firstly modulated as multiple sets of marginals and then stacked into $\mathbf{U}\in\mathbb{R}_{\geq 0}^{k\times (l+1)}, \mathbf{V}\in\mathbb{R}_{\geq 0}^{k\times 2}$, where $k$ is the number of constraints. By building $\mathbf{W}$ from $\mathbf{y}$, getting $\mathbf{S} = \exp(\mathbf{W}/\tau)$ and calling \cref{alg:extended-sk} based on $\mathbf{S}, \mathbf{U}, \mathbf{V}$, the satisfiability of positive linear constraints is enforced to the output of neural networks.

\textbf{Implementation Details}. We set separate dummy variables for different constraints to handle potential conflicts among different sets of marginals (see explanations in \cref{sec:feasibility-assump}).

\section{Case Study I: Neural Solver for Traveling Salesman Problem with Extra Constraints}%\vspace{-4pt}
\label{sec:tsp}
% one-page
\subsection{Problem Background}%\vspace{-4pt}
The Traveling Salesman Problem (TSP) is a classic NP-hard problem. The standard TSP aims at finding a cycle visiting all cities with minimal length, and developing neural solvers for TSP receives increasing interest~\citep{VinyalsNIPS15,kool2018attention,kwon2021matrix}. Beyond standard TSP, here we develop a neural solver for TSP with extra constraints using {\mname} layer.

\begin{table*}[tb!]
  \centering
  \caption{Comparison of average tour length and total inference time for 10,000 testing instances on TSP variants with extra constraints. ``Standard Solver'' means state-of-the-art solvers for standard TSP. Our method is marked as gray.}
  \resizebox{0.65\linewidth}{!}
{
    \begin{tabular}{cr|cc|cc}
    \toprule
     & & \multicolumn{2}{c|}{TSP-SE} & \multicolumn{2}{c}{TSP-PRI}\\
    & Method & Tour Length & Time & Tour Length & Time\\
    \midrule
    \multirow{2}{*}{MIP} & Gurobi (Sec.\ref{sec:TSP_fml}, 2s) & 4.608 & 5h34m & 4.720 & 5h34m\\
    & Gurobi (Sec.\ref{sec:TSP_fml}, 10s) & 4.010 & 27h44m & 4.148 & 27h45m\\
    \midrule
    \multirow{4}{*}{Heuristic} & Nearest Neighbor & 4.367 & 0s & 4.674 & 0s\\
    & Nearest Insertion &  4.070 & 9s & 4.349 & 9s\\
    & Farthest Insertion &	3.772 &	11s & 4.403 & 10s\\
    & Random Insertion &	3.853 &	5s & 4.469 & 4s\\
    \midrule
    \multirow{3}{*}{\thead{Standard \\ Solver}} & Gurobi (MTZ) & \textbf{3.648} & 1h2m & - & -\\
    & Concorde & \textbf{3.648} & 9m28s & - & -\\
    & LKH3 & \textbf{3.648} & 2m44s & - & -\\
    \midrule
    \multirow{2}{*}{\thead{Neural}} & Attention Model &  3.677 & 4m39s & 4.008 & 4m43s\\
    % \rowcolor{gray!40}  & {\mname} (ours) & 3.811 & 19s & \textbf{3.943} & 18s\\
      & \cellcolor{gray!40} {\mname} (ours) & \cellcolor{gray!40} 3.811 & \cellcolor{gray!40} 19s & \cellcolor{gray!40} \textbf{3.943} & \cellcolor{gray!40} 18s\\
    \bottomrule
    \end{tabular}
    }%
  \label{tab:TSP}%
\end{table*}%

\subsection{Constraint Formulation for {\mname}} %\vspace{-4pt}
\label{sec:TSP_fml}
We consider 1)~TSP with starting and ending cities constraint (TSP-SE); 2)~TSP with priority constraint (TSP-PRI). 

\textbf{1) TSP-SE.} We little abuse notations that appeared in Sec.~\ref{sec:method}. Given $n$ cities and two of them are the starting and ending cities $s,e \in \{1, \dots, n\}$. The distance matrix $\mathbf{D} \in \mathbb{R}^{n \times n}_{\ge 0}$ records the distances between city pairs. TSP-SE finds the shortest tour starting from city $s$, visiting other cities exactly once, and ending in city $e$. TSP-SE can be formulated with the following objective function and constraints:
\begin{subequations} \label{eq:SE_Obj}
    \begin{align}
        \min_\mathbf{X} \quad &\sum_{i = 1}^{n}\sum_{j = 1}^{n}D_{i,j}\sum_{k = 1}^{n - 1}X_{i,k}X_{j, k + 1}, \tag{\ref{eq:SE_Obj}}\\
        \text{s.t. } &\sum_{i = 1}^n X_{i,k} = 1, \forall k \in \{1, \dots, n\}, \label{eq:birkhoff_1}\\
        &\sum_{k = 1}^n X_{i,k} = 1, \forall i \in \{1, \dots, n\}, \label{eq:birkhoff_2}\\
        &X_{s,1} = 1, \quad X_{e,n} = 1, \label{eq:start_End}\\
        &X_{i,k} \in \{0,1\}, \quad \forall i,j \in \{1, \dots, n\}, \label{eq:binary}
    \end{align}
\end{subequations}
where $\mathbf{X} \in \{0,1\}^{n \times n}$ is a binary matrix and $X_{i,k} = 1$ indicates city $i$ is the $k$-th visited city in the tour. Constraints~(\ref{eq:birkhoff_1}) and (\ref{eq:birkhoff_2}) ensure $\mathbf{X}$ to be a valid tour and  constraint~(\ref{eq:start_End}) defines the starting and ending cities. If $X_{i,k}X_{j, k + 1} = 1$ for some $k$, then the $k$-th step of the tour is from $i$ to $j$, and $D_{i,j}$ will be counted into the objective.

\textbf{2) TSP-PRI.} In practice, some cities may have higher priority and need to be visited earlier. In TSP-PRI we consider: in the given $n$ cities, the priority city $p \ne s,e$ has to be visited within the first $m$ steps. We add a new constraint to TSP-SE to formulate TSP-PRI:
\begin{equation}
    \label{eq:priority}
    \sum_{k = 1}^{m+1} X_{p,k} = 1.
\end{equation}
To fit with the continuous nature of neural networks, we relax the binary constraint (\ref{eq:binary}) to continuous ones $\widetilde{X}_{i,k} \in [0, 1]$ which is automatically satisfied by {\mname}. A neural network takes the instance as input and outputs the pre-projected matrix $\mathbf{Y} \in \mathbb{R}^{n \times n}$. $\mathbf{Y}$ is flattened into a $n^2$-dimensional vector and projected via {\mname} to enforce all the aforementioned constraints. Note that the neural network itself is a solver to an optimization problem, enforcing the constraint satisfiability by {\mname} is a reasonable choice instead of optimizing some other auxiliary objectives.

\subsection{Network Design Details}%\vspace{-4pt}
Following the Attention Model for standard TSP~\citep{kool2018attention}, we use a Transformer~\citep{vaswani2017attention} without positional encoding to encode each of the $n$ nodes into a hidden vector $\mathbf{h}_i$. Learnable embeddings to mark starting, ending and priority cities are added to the corresponding embeddings before input to the Transformer.  After encoding, $\mathbf{h}_i$ is projected into $\mathbf{Y}_i \in \mathbb{R}^n$ using an MLP. All $\mathbf{Y}_i, i \in \{1,\dots,n\}$ form the pre-projected matrix $\mathbf{Y} \in \mathbb{R}^{n \times n}$. In training, the objective Eq.~(\ref{eq:SE_Obj}) with continuous $\widetilde{\mathbf{X}}$ as the decision variable is used as the unsupervised loss. For inference, we first output $\widetilde{\mathbf{X}}$. As $\widetilde{\mathbf{X}}$ satisfies constraints~(\ref{eq:birkhoff_1}) and (\ref{eq:birkhoff_2}), it can be viewed as the marginal distributions of the binary $\mathbf{X}$~\citep{AdamsArxiv11}. We perform beam search on $\widetilde{\mathbf{X}}$ to get $\mathbf{X}$ in post-processing.

\subsection{Experiments}%\vspace{-4pt}
Following~\citet{kool2018attention}, for both TSP variants, we generate 10,000 2-D Euclidean TSP instances as the testing set. Each instance consists of $n = 20$ nodes uniformly sampled in the unit square $[0,1]^2$. The starting, ending, and priority cities are randomly selected. As our model is unsupervised, the training set is generated on the fly using the same process. For TSP-PRI, the number of priority steps is set to $m = 5$.  The following baselines are considered with results shown in \cref{tab:TSP}: 1)~Mixed integer programming (MIP) solvers by directly applying the commercial solver Gurobi~\citep{llc2020gurobi} to formulations in Sec.~\ref{sec:TSP_fml} and the time limit per instance is set as 2s/10s; 2)~Heuristics e.g.\ nearest neighbor and insertion heuristics~\citep{johnson1990local} that are usually fast and approximate algorithms; 3)~State-of-the-art solvers for standard TSP like Concorde\footnote{https://www.math.uwaterloo.ca/tsp/concorde/index.html} and LKH3~\cite{helsgaun2017extension}, and Gurobi (MTZ) means applying Gurobi to the TSP formulation named after Miller-Tucker-Zemlin (MTZ)~\citep{miller1960integer}; 4)~RL-based neural routing solver Attention Model~\citep{kool2018attention}.

Because the objective in Eq.~(\ref{eq:SE_Obj}) is quadratic w.r.t.\ $\textbf{X}$, it is hard for a MIP solver to get a satisfactory solution quickly. Heuristic methods run much faster and perform well on TSP-SE, but their performances drop greatly on TSP-PRI. TSP-SE can be converted to the standard TSP, making it possible for standard solvers to find the optimal tour within a reasonable time. However, these highly specialized methods cannot be easily transferred to TSP-PRI. The RL-based Attention Model performs worse than our {\mname} on TSP-PRI as its performance highly depends on specialized decoding strategies for different tasks.  Finally, our {\mname} can get near-optimal solutions in a short time on both TSP variants, and it is easy to transfer from TSP-SE to TSP-PRI by adding one single constraint. Details of TSP experiments are discussed in Appendix~\ref{sec:detailed_tsp_exp}.

% See Fig.~\ref{fig:visual} for an example.
\section{Case Study II: Partial Graph Matching with Outliers on Both Sides}%\vspace{-4pt}
\label{sec:gm}
% one-pageJiangMM21
\subsection{Problem Background}%\vspace{-4pt}
% graph matching and partial graph matching
Standard graph matching (GM) assumes an outlier-free setting namely bijective mapping. One-shot GM neural networks~\citep{WangPAMI22} effectively enforce the satisfiability of one-to-one matching constraint by single-set Sinkhorn (\cref{alg:sk}). Partial GM refers to the realistic case with outliers on both sides so that only a partial set of nodes are matched. There lacks a principled approach to enforce matching constraints for partial GM. \yanr{The main challenge for existing GM networks is that they cannot discard outliers because the single-set Sinkhorn is outlier-agnostic and tends to match as many nodes as possible. The only exception is BBGM~\citep{RolinekECCV20} which incorporates a traditional solver that can reject outliers, yet its performance still has room for improvement.}

\subsection{Constraint Formulation for {\mname}}%\vspace{-4pt}
\label{sec::pgm_constraints}
% x, A, b
Denote a graph pair by $\mathcal{G}_1=(\mathcal{V}_1, \mathcal{E}_1)$, $\mathcal{G}_2=(\mathcal{V}_2, \mathcal{E}_2)$, where $|\mathcal{V}_1|=n_1$, $|\mathcal{V}_2|=n_2$. In mainstream GM networks, a matching score matrix $\mathbf{M} \in \mathbb{R}^{n_1 \times n_2}$ is expected to describe the correspondences of nodes between $\mathcal{G}_1$ and $\mathcal{G}_2$, where $\mathbf{M}_{i,j}$ refers to the matching score between node $i$ in $\mathcal{G}_1$ and node $j$ in $\mathcal{G}_2$. In previous bijective GM networks, the one-to-one node matching constraint that a node corresponds to at most one node is enforced by the off-the-shelf Sinkhorn algorithm in \cref{alg:sk}.
It cannot take the outliers into consideration, as it forcibly matches all nodes. The partial GM problem can be formulated by adding a partial matching constraint: assume that the number of inliers is $\phi$, so the number of matched nodes should not exceed $\phi$.

With a little abuse of notations, denote $\mathbf{X}\in [0,1]^{n_1\times n_2}$ as the output of our partial GM network, the partial GM problem has the following constraints,
\begin{subequations}
% \vspace{-15pt}
\label{eq:pgm_constraint}
\begin{align}
    &\sum_{i=1}^{n_1} X_{i,j} \leq 1, \forall j\in \{1, \dots, n_2\}, \label{eq:pgm-row} \\
    &\sum_{j=1}^{n_2} X_{i,j} \leq 1, \forall i\in \{1, \dots, n_1\}, \label{eq:pgm-col} \\
    &\sum_{i=1}^{n_1}\sum_{j=1}^{n_2} X_{i,j} \leq \phi. \label{eq:pgm-topk}
\end{align}
\end{subequations}
The constraint~(\ref{eq:pgm-row}) and (\ref{eq:pgm-col}) denotes the node-matching on rows and columns, respectively, and they ensure (at most) one-to-one node correspondence. Constraint~(\ref{eq:pgm-topk}) is the partial matching constraint ensuring that the total number of matched node pairs should not exceed $\phi$. All constraints are positive linear and can be enforced by {\mname} layer. We implement our partial GM neural network by flattening $\mathbf{M}$ into a $n_1n_2$-dimensional vector to feed into {\mname}.

\begin{table}[tb!]
  \centering
  \caption{F1 (\%) on Pascal VOC Keypoint (unfiltered setting). ``Sinkhorn'' denotes the classic single-set version in \cref{alg:sk}.}
  \resizebox{\linewidth}{!}{
    \begin{tabular}{r|c|c|c}
    \toprule
    GM Net & Constraint Technique & Matching Type & Mean F1 \\
    \midrule
    PCA-GM & Sinkhorn & bijective  & 48.6  \\
    BBGM  & \citep{PoganvcicICLR19} & bijective  & 51.9  \\
    NGMv2 & Sinkhorn & bijective & 58.8  \\
    \midrule
    BBGM  & \citep{PoganvcicICLR19} & partial  & 59.0  \\
    NGMv2 & Sinkhorn+post-processing & partial & 60.7 \\
    \rowcolor{gray!40} NGMv2 & {\mname} (ours) & partial & \textbf{61.2} \\
    \bottomrule
    \end{tabular}
    }%
    \vspace{-15pt}
  \label{tab:voc}%
\end{table}%

\subsection{Network Design Details}%\vspace{-4pt}
\label{sec:gm_network}
We follow the SOTA GM network NGMv2~\citep{WangPAMI22} and replace the original Sinkhorn layer with {\mname} to tackle the partial GM problem on natural images.
Specifically, a VGG16~\citep{simonyanICLR14vgg} network is adopted to extract initial node features and global features from different CNN layers. The node features are then refined by SplineConv~\citep{FeyCVPR18}. The edge features are produced by the node features and the connectivity of graphs. The matching scores are predicted by the neural graph matching network proposed by \citet{WangPAMI22}, finally generating the matching scores $\mathbf{M}$. We replace the original single-set Sinkhorn layer by {\mname} to enforce the constraints in Eq.~(\ref{eq:pgm_constraint}). The output of {\mname} is reshaped into matrix $\mathbf{\hat{M}}$, which is used for end-to-end training with permutation loss~\citep{WangICCV19}.
During inference, the Hungarian algorithm~\citep{KuhnNavalResearch55} is performed on $\mathbf{\hat{M}}$ and we retain the $\phi$-highest matching scores from $\mathbf{\hat{M}}$, and the remaining matches are discarded.

\subsection{Experiments}
%\vspace{-4pt}
% Pascal VOC and baselines
We do experiments on Pascal VOC Keypoint dataset~\citep{Everingham10Pascal} with Berkeley annotations~\citep{bourdev2009poselets_VOCkeypoint} under the ``unfiltered'' setting following \citet{RolinekECCV20} and report the matching F1 scores between graph pairs. We assume that the number of inliers $\phi$ is given (e.g.\ estimated by another regression model) and focus on the GM networks. As there are no one-shot partial GM networks available, we compare with bijective matching networks: PCA-GM~\citep{WangICCV19} and BBGM~\citep{RolinekECCV20}. We also build a partial GM baseline by post-processing  retaining only the top-$\phi$ matches. Table~\ref{tab:voc} shows that our method performs the best. Note that BBGM (matching type=bijective) is an example of applying black-box solvers to an ill-posed optimization problem because the objective function does not consider the outliers, leading to inferior performance.

\begin{table}[tb!]
  \centering
    % \vspace{-10pt}
  \caption{Sharpe ratio of portfolio allocation methods. The constraint technique ``Gurobi Opt'' means solving a constrained optimization problem by the commercial solver Gurobi whereby the optimization parameters are based on the predicted asset prices.}
  \resizebox{\linewidth}{!}{
    \begin{tabular}{r|c|c|c}
    \toprule
    Predictor & Constraint Technique & Expert Pref.? & Mean Sharpe \\
    \midrule
    LSTM & Softmax  & No & 2.15  \\
    LSTM  & Gurobi Opt & Yes  & 2.08  \\
    \rowcolor{gray!40} LSTM & {\mname} (ours) & Yes & \textbf{2.27}  \\    \midrule
    StemGNN & Softmax  & No & 2.11  \\
    StemGNN  & Gurobi Opt & Yes  & 2.00  \\
    \rowcolor{gray!40} StemGNN & {\mname} (ours) & Yes & \textbf{2.42}  \\
    \bottomrule
    \end{tabular}
    }%
    \vspace{-15pt}
  \label{tab:portfolio}%
\end{table}%

\section{Case Study III: Portfolio Allocation}
\label{sec:portfolio}
% one-page
\subsection{Problem Background}%\vspace{-4pt}
Predictive portfolio allocation is the process of selecting the best asset allocation based on predictions of future financial markets. The goal is to design an allocation plan to best trade-off between the return and the potential risk (i.e.\ the volatility). In an allocation plan, each asset is assigned a non-negative weight and all weights should sum to 1. Existing learning-based methods \citep{zhang2020deep,butler2021integrating} only consider the sum-to-one constraint without introducing personal preference or expert knowledge. In contrast, we achieve such flexibility for the target portfolio via positive linear constraints: a mix of covering and equality constraints, which is widely considered \citep{sharpe1971,mansini2014twenty} for its real-world demand.

\subsection{Constraint Formulation for {\mname}}%\vspace{-4pt}
Given historical data of assets, we aim to build a portfolio whose future Sharpe ratio~\citep{sharpe1998sharpe} is maximized. 
$
    \text{Sharpe ratio}=\frac{\text{return} - \text{r}_\text{f}}{\text{risk}}
$,
where $\text{r}_\text{f}$ denotes the risk-free return and is assumed to be 3\% (annually).
Besides the sum-to-one constraint, we consider the extra constraint based on expert preference: among all assets, the proportion of assets in set $\mathcal{C}$ should exceed $p$. This is reasonable as some assets (e.g.\ tech giants) have higher Sharpe ratios than others in certain time periods. Formally, the constraints are formulated as:
\begin{align}
    \sum_{i=1}^{n} x_{i} = 1,  \quad
    \sum_{i\in\mathcal{C}} x_i \geq p,
 \label{eq:port}
\end{align}
where $\mathbf{x}\in [0,1]^{n}$ is the predicted portfolio. The first constraint is the traditional sum-to-one constraint and the second one is the extra preference constraint.

\subsection{Network Design Details}%\vspace{-4pt}
We adopt LSTM~\citep{hochreiter1997long} and StemGNN~\citep{cao2020spectral} as two variants of portfolio allocation networks for their superiority in learning with time series. Our network has two output branches, one predicts future asset prices and the other predicts the portfolio. {\mname} is applied to the portfolio prediction branch to enforce constraints in Eq.~(\ref{eq:port}). The network receives supervision signals by a weighted sum of maximizing the Sharpe ratio and minimizing the prediction error on future asset prices (based on the historical data in the training set).

\subsection{Experiments}%\vspace{-4pt}
We consider the portfolio allocation problem where the network is given the historical data in the previous 120 trading days, and the goal is to build a portfolio with maximized Sharpe ratio for the next 120 trading days. The training set is built on the real prices of 494 assets from the S\&P 500 index from 2018-01-01 to 2020-12-30, and the models are tested on real-world data from 2021-03-01 to 2021-12-30. Without loss of generality, we impose the expert preference that in the period of interest, the following tech giants' stocks could be more profitable: $\mathcal{C} = \{\textit{AAPL, MSFT, AMZN, TSLA, GOOGL, GOOG}\}$, and the preference ratio is set to $p = 50\%$.

We build two baselines: 1)~A neural network portfolio allocator without preference, and the sum-to-one constraint is enforced by softmax following \citet{zhang2020deep}; 2)~A two-stage allocator that first predicts future prices and then uses Gurobi~\citep{llc2020gurobi} to solve a constrained optimization problem whose objective function is based on the predicted prices. See results in \cref{tab:portfolio}. Compared with an allocator without preference, the expert preference information improves the performance; Compared with the two-stage allocator, our allocator reduces the issue of error accumulation and builds better portfolios. Note that the objective function in the two-stage allocator is ill-posed because the first-stage prediction unavoidably contains errors.

\section{Conclusion and Outlook}%\vspace{-4pt}
We have presented {\mname}, a principled approach to enforce the satisfiability of positive linear constraints for the solution as predicted in one-shot by neural network. The satisfiability layer is built upon an extended Sinkhorn algorithm for multi-set marginals, whose convergence is theoretically characterized. We showcase three applications of {\mname}. Future work may be improving the efficiency of both forward and backward of {\mname}.

\section*{Acknowledgments}%\vspace{-4pt}
The work was supported in part by National Key Research and Development Program of China (2020AAA0107600), NSFC (62222607, U19B2035), and Science and Technology Commission of Shanghai Municipality (22511105100).

\bibliography{example_paper}
\bibliographystyle{icml2023}

%%%%%%%%%%%%%%%%%%%%%%%%%%%%%%%%%%%%%%%%%%%%%%%%%%%%%%%%%%%%%%%%%%%%%%%%%%%%%%%
%%%%%%%%%%%%%%%%%%%%%%%%%%%%%%%%%%%%%%%%%%%%%%%%%%%%%%%%%%%%%%%%%%%%%%%%%%%%%%%
% APPENDIX
%%%%%%%%%%%%%%%%%%%%%%%%%%%%%%%%%%%%%%%%%%%%%%%%%%%%%%%%%%%%%%%%%%%%%%%%%%%%%%%
%%%%%%%%%%%%%%%%%%%%%%%%%%%%%%%%%%%%%%%%%%%%%%%%%%%%%%%%%%%%%%%%%%%%%%%%%%%%%%%
\newpage
\appendix
\onecolumn

% \section{Explaining the Difference between \emph{Optimization} and \emph{Decision}}

% In the first paragraph of this paper, we categorize constrained learning problems into \emph{optimization} and \emph{decision} problems ({\mname} belongs to \emph{decision} problems). In case the readers have any confusions between these two concepts, we would like to elaborate on their differences:
% \begin{itemize}
%     \item \emph{optimization}-based methods consider explicit objective functions that are directly related to downstream tasks. The optimization forms are usually more complicated, whose inputs are usually problem parameters. For example, \citet{PoganvcicICLR19} directly solve TSP by Gurobi and the pair-wise distances are predicted by a neural network. 
    
%     \item \emph{decision}-based methods may be viewed as projections that accept arbitrary input and project it to the feasible region. They do have objective functions, however, their objectives are usually interpreted as ``finding a feasible solution nearest to the input''. \textbf{The objective of the downstream task is agnostic to \emph{decision} methods, or the downstream task may not have any explicit objectives.} For example, in our \cref{sec:tsp}, the {\mname} layer is unaware of the TSP objective, and its only target is to enforce the constraints. \citet{CruzCVPR17} predict permutations to solve visual jigsaw puzzles, while jigsaw puzzles have no explicit objective. 
% \end{itemize}

%   We hope the above clarification could resolve your confusion. We will also add these discussions in our revision.

\section{Comparison with the Notations from \citet{CuturiNIPS13}}
\label{sec:discuss_with_classic_sinkhorn}

The formulation used in this paper (regarding $\mathbf{\Gamma}$) is an equivalent adaptation from the notations used in existing single-set Sinkhorn papers e.g. \citet{CuturiNIPS13}. As we explained in the footnote in page 3, this new formulation is preferred as we are generalizing the scope of Sinkhorn to multi-set marginal, and the existing formulation cannot seamlessly handle marginals with different values.

Specifically, we make a side-by-side comparison with the notations used in this paper and the notations used in \citet{CuturiNIPS13} on single-set Sinkhorn algorithm:
\begin{itemize}
    \item This paper's notations:
    
    The transportation matrix is $\mathbf{\Gamma}\in[0,1]^{m\times n}$, and the constraints are
    \begin{align}
        \sum_{i=1}^m \Gamma_{i,j}u_{j}=u_{j},\quad
        \sum_{j=1}^n\Gamma_{i,j}u_{j}=v_{i}.
    \end{align}

  $\Gamma_{i,j}$ means the \emph{proportion} of $u_{j}$ moved from $u_{j}$ to $v_{i}$.

  The algorithm steps are:

    $\quad$\textbf{repeat}: 

    $\qquad{\Gamma}_{i,j}^{\prime} = \frac{{\Gamma}_{i,j}v_{i}}{\sum_{j=1}^n {\Gamma}_{i,j}u_{j}}$; $\triangleright$ normalize w.r.t. $\mathbf{v}$

    $\qquad{\Gamma}_{i,j} = \frac{{\Gamma}_{i,j}^{\prime}u_{j}}{\sum_{i=1}^m {\Gamma}_{i,j}^{\prime}u_{j}}$; $\triangleright$ normalize w.r.t. $\mathbf{u}$
    
    $\quad$\textbf{until} convergence.
    
    \item \citet{CuturiNIPS13}'s notations:
    
    The transportation matrix is $\mathbf{P}\in\mathbb{R}_{\geq 0}^{m\times n}$, and the constraints are 
    \begin{align}
     \sum_{i=1}^m P_{i,j}=u_{j},\quad
     \sum_{j=1}^n P_{i,j}=v_{i}.
    \end{align}
  $P_{i,j}$ means the \emph{exact mass} moved from $u_{j}$ to $v_{i}$.

  The algorithm steps are:

    $\quad$\textbf{repeat}: 

     $\qquad{P}_{i,j}^{\prime} = \frac{P_{i,j}v_{i}}{\sum_{j=1}^n {P}_{i,j}}$; $\triangleright$ normalize w.r.t. $\mathbf{v}$

     $\qquad{P}_{i,j} = \frac{{P}_{i,j}^{\prime}u_j}{\sum_{i=1}^m {P}_{i,j}^{\prime}}$; $\triangleright$ normalize w.r.t. $\mathbf{u}$
    
    $\quad$\textbf{until} convergence.

\end{itemize}

The equivalence between the above formulations becomes clear if we substitute $P_{i,j}$ by $\Gamma_{i,j}u_j$ and $P^\prime_{i,j}$ by $\Gamma^\prime_{i,j}u_j$ in all the above definitions and algorithm steps. We would like to highlight that such different notations are necessary because making $\Gamma_{i,j}$ as the proportion of all $u_{1,j}, u_{2,j},\cdots,u_{k,j}$ could deal with multiple marginals with different values, making it generalizable to the multi-marginal case.

\section{Proof of \cref{thm:extended-sk}.}
\label{sec:detailed_proof}

To prove the upper bound of convergence rate, we define the Kullback-Leibler (KL) divergence for matrices $\mathbf{Z}$ and $\mathbf{\Gamma}$, whereby the KL divergence is originally defined for probability vectors, 
\begin{equation}
    D(\mathbf{Z}, \mathbf{\Gamma}, \eta)=\frac{1}{h_\eta}\sum_{i=1}^m\sum_{j=1}^n z_{i,j} u_{\eta,j} \log \frac{z_{i,j}}{\Gamma_{i,j}}.
\end{equation}
Also recall that $\mathbf{Z}\in[0,1]^{m\times n}$ is a normalized matrix satisfying all marginal distributions,
\begin{equation}
    \forall \eta\in \{1,\cdots, k\}: \sum_{i=1}^m z_{i,j} u_{\eta,j}=u_{\eta,j}, \sum_{j=1}^n z_{i,j} u_{\eta,j}=v_{\eta,i},
\end{equation}

We then prove the convergence rate w.r.t.\ the KL divergence based on the following two Lemmas.

\textbf{\cref{lemma:kl_gamma0}.}
For any $\eta=1,\cdots,k$, $D(\mathbf{Z}, \mathbf{\Gamma}^{(0)},\eta) \leq \log(1+2\Delta/\alpha)$.
\begin{proof}
By definition, 
\begin{align}
    D(\mathbf{Z}, \mathbf{\Gamma}^{(0)}, \eta)=\frac{1}{h_\eta}\sum_{j=1}^n u_{\eta,j} \sum_{i=1}^m z_{i,j} \log \frac{z_{i,j}}{\Gamma_{i,j}},
\end{align}
where $\{z_{i,j}\}_{i=1,2,\cdots,m}$ and $\{\Gamma_{i,j}\}_{i=1,2,\cdots,m}$ sum to 1 thus they are probability distributions. The second summand is a standard KL divergence term.

Based on the following fact in terms of Eq.~(27) from \citet{sason2015upper}, for probability distributions $\mathbf{p},\mathbf{q}$,
\begin{align}
    D_{\text{KL}}(\mathbf{p} || \mathbf{q})&\leq \log\left(1+\frac{\|\mathbf{p}-\mathbf{q}\|_2^2}{q_{min}}\right)\label{eq:kl-upperbound-1} \\ 
    &\leq \log\left(1+\frac{2}{q_{min}}\right), \label{eq:kl-upperbound-2}
\end{align}
where Eq.~(\ref{eq:kl-upperbound-1}) to Eq.~(\ref{eq:kl-upperbound-2}) is because $\|\mathbf{p}-\mathbf{q}\|_2\leq \sqrt{2}$, and $q_{min}$ denotes the smallest non-zero element in $\mathbf{q}$. We have the following conclusion for $\mathbf{Z}$ and $\mathbf{\Gamma}^{(0)}$,
\begin{equation}
\begin{aligned}
    D(\mathbf{Z},\mathbf{\Gamma}^{(0)},\eta) &\leq \frac{1}{h_\eta}\sum_{j=1}^n u_{\eta,j} \log\left(1+\frac{2}{\Gamma^{(0)}_{min}}\right) \\
    &= \log\left(1+\frac{2}{\Gamma^{(0)}_{min}}\right),
\end{aligned}
\end{equation}
where $\frac{1}{h_\eta}\sum_{j=1}^n u_{\eta,j}=1$ by definition.

Recall that $\alpha={\min_{i,j:s_{i,j}>0} s_{i,j}}/{\max_{i,j} s_{i,j}}$, $\Delta=\max_j \left|\{i: s_{i,j} > 0\}\right|$ is the max number of non-zeros in any column of $\mathbf{S}$, and in our algorithm $\Gamma_{i,j}^{(0)}=\frac{s_{i,j}}{\sum_{i=1}^m s_{i,j}}$, we have
\begin{align}
    &\Gamma_{min}^{(0)} \geq \frac{\alpha}{\Delta}\\
    \Rightarrow \quad & D(\mathbf{Z},\mathbf{\Gamma}^{(0)},\eta) \leq \log\left(1+\frac{2\Delta}{\alpha}\right).
\end{align}

This ends the proof of \cref{lemma:kl_gamma0}.
\end{proof}

\textbf{\cref{lemma:kl_diff}.}
For $\eta = (t\mod k)+1$, $\eta^\prime = (t+1\mod k)+1$, we have
\begin{align}
    D(\mathbf{Z}, \mathbf{\Gamma}^{(t)},\eta) - D(\mathbf{Z},\mathbf{\Gamma}^{\prime(t)},\eta) &= D_{\text{KL}}(\pi_{\mathbf{v}_\eta}|| \pi_{\mathbf{v}^{(t)}_\eta})\\
    D(\mathbf{Z}, \mathbf{\Gamma}^{\prime(t)},\eta) - D(\mathbf{Z}, \mathbf{\Gamma}^{(t+1)},\eta^\prime)& = D_{\text{KL}}(\pi_{\mathbf{u}_\eta}||\pi_{\mathbf{u}^{(t)}_\eta}) 
\end{align}
\begin{proof}
By definition, the left-hand side of the first inequality is
\begin{equation}
\begin{aligned}
    D(\mathbf{Z}, \mathbf{\Gamma}^{(t)},\eta) - D(\mathbf{Z}, \mathbf{\Gamma}^{\prime(t)},\eta) &= \frac{1}{h_\eta} \sum_{i=1}^m\sum_{j=1}^n z_{i,j} u_{\eta,j} \log \frac{z_{i,j}}{\Gamma_{i,j}^{(t)}} - \frac{1}{h_\eta} \sum_{i=1}^m\sum_{j=1}^n z_{i,j} u_{\eta,j} \log \frac{z_{i,j}}{\Gamma^{\prime(t)}_{i,j}}\\
    &= \frac{1}{h_\eta}\sum_{i=1}^m\sum_{j=1}^n z_{i,j}u_{\eta,j} \log \frac{\Gamma_{i,j}^{\prime(t)}}{\Gamma_{i,j}^{(t)}}\\
    &= \frac{1}{h_\eta}\sum_{i=1}^m\sum_{j=1}^n z_{i,j}u_{\eta,j} \log \frac{v_{\eta,i}}{\sum_{j=1}^n {\Gamma}_{i,j}^{(t)}u_{\eta,j}} \\
    &= \frac{1}{h_\eta}\sum_{i=1}^m\sum_{j=1}^n z_{i,j} u_{\eta,j} \log \frac{v_{\eta,i}}{v_{\eta,i}^{(t)}} \\
    &= \frac{1}{h_\eta}\sum_{i=1}^m \log \frac{v_{\eta,i}}{v_{\eta,i}^{(t)}} \sum_{j=1}^n z_{i,j} u_{\eta,j}\\
    &= \sum_{i=1}^m \frac{v_{\eta,i}}{h_\eta}\log \frac{v_{\eta,i}/h_\eta}{v_{\eta,i}^{(t)}/h_\eta}
\end{aligned}
\end{equation}
which is exactly $D_{\text{KL}}(\pi_{\mathbf{v}_\eta}|| \pi_{\mathbf{v}^{(t)}_\eta})$. The other equation could be derived analogously. This ends the proof of \cref{lemma:kl_diff}.
\end{proof}

Denote $T=\frac{k\log(1+2\Delta/\alpha)}{\delta}+\zeta$, where $\zeta \in [0, k)$ is a residual term ensuring $(T+1)\mod k = 0$. If all $D_{\text{KL}}(\pi_{\mathbf{v}_\eta}|| \pi_{\mathbf{v}^{(t)}_\eta}) > \delta$ and $D_{\text{KL}}(\pi_{\mathbf{u}_\eta}|| \pi_{\mathbf{u}^{(t)}_\eta}) > \delta$ for all $\eta$, by substituting \cref{lemma:kl_diff} and summing, we have
\begin{equation}
    D(\mathbf{Z}, \mathbf{\Gamma}^{(0)},1) -  D(\mathbf{Z}, \mathbf{\Gamma}^{(T+1)},1) > \frac{T\delta}{k} \geq \log(1+2\Delta/\alpha).
\end{equation}
The multiplier $k$ exists because we need to sum over $k$ sets of marginals to cancel all intermediate terms. Since KL divergence must be non-negative, the above formula contradicts with \cref{lemma:kl_gamma0}. This ends the proof of \cref{thm:extended-sk}. \qed

\section{Further Discussions with the Entropic Regularizer}
\label{sec:discuss_with_tau}

In the main paper, we write our algorithms after the regularization term for simplicity. On one hand, the entropic regularizer may be omitted if the score matrix is non-negative (e.g.\ activated by ReLU). On the other hand, our theoretical insights could naturally generalize with the regularizer. We provide the detailed discussions as follows.

\subsection{The Underlying Formulation of \cref{alg:extended-sk}}
\label{sec:opt-form}

  Recall that given real-valued matrix $\mathbf{W}\in \mathbb{R}^{m\times n}$, regularizer $\tau$, and a set of target marginals $\mathbf{u}_\eta \in \mathbb{R}^n_{\geq 0}$, $\mathbf{v}_\eta \in \mathbb{R}^m_{\geq 0}$, our multi-set marginal Sinkhorn algorithm maps $\mathbf{W}$ to $\mathbf{\Gamma}\in[0,1]^{m\times n}$ such that $\forall \eta \in \{1,\cdots,k\}: \sum_{i=1}^m \Gamma_{i,j}u_{\eta,j}=u_{\eta,j}, \sum_{j=1}^n\Gamma_{i,j}u_{\eta,j}=v_{\eta,i}$. In the following, we discuss the underlying formulation for a special (but general enough) case, when the values of $u_{\eta,j}$ are binary.

  Formally, if $u_{\eta,j}$ is binary, i.e.\ either $u_{\eta,j}=0$ or $u_{\eta,j}=c$ for all $\eta, j$, the following entropic regularized problem is considered:
  \begin{subequations}
  \begin{align}
        & \min_\mathbf{\Gamma} \ -tr(\mathbf{W}^\top\mathbf{\Gamma})+ \tau\sum_{i,j}\Gamma_{i,j}\log\Gamma_{i,j}, \\
         s.t. \quad &\mathbf{\Gamma}\in[0,1]^{m\times n}, \\ 
          & \forall \eta \in \{1,\cdots,k\}: \sum_{i=1}^m \Gamma_{i,j}u_{\eta,j}=u_{\eta,j}, \sum_{j=1}^n\Gamma_{i,j}u_{\eta,j}=v_{\eta,i}.
  \end{align}
  \end{subequations}

  With dual variables $\theta_{\eta,u}\in\mathbb{R}^n,\theta_{\eta,v}\in\mathbb{R}^m$, the Lagrangian of the above problem is
  \begin{equation}
      \mathcal{L}(\Gamma,\theta_{*})=\sum_{i,j}\tau\Gamma_{i,j}\log \Gamma_{i,j}-w_{i,j}\Gamma_{i,j}+\sum_{\eta=1}^k\theta^\top_{\eta,v} (\mathbf{\Gamma}\mathbf{u}_\eta - \mathbf{v}_\eta) + \theta^\top_{\eta,u}(\mathbf{\Gamma}^\top\mathbf{1}_m - \mathbf{1}_n),
  \end{equation}
  and for any ($i,j$),
  \begin{subequations}
  \begin{align}
       & \frac{\partial\mathcal{L}}{\partial \Gamma_{i,j}}=0\\
       \Rightarrow \quad & \tau\log\Gamma_{i,j}+\tau-w_{i,j}+\sum_{\eta=1}^k\theta_{\eta,v,i}u_{\eta,j}+\theta_{\eta,u,j}=0\\
       \Rightarrow \quad & \Gamma_{i,j} = \exp(-\frac{1}{\tau}\sum_{\eta=1}^k\theta_{\eta,v,i}u_{\eta,j})\exp(\frac{w_{i,j}}{\tau})\exp(-1-\frac{1}{\tau}\sum_{\eta=1}^k\theta_{\eta,u,j})
  \end{align}
  \end{subequations}

Since all $u_{\eta,j}$ are binary, and the corresponding elements will not be normalized if $u_{\eta},j=0$ in our implementation of \cref{alg:extended-sk}, $\mathbf{\Gamma}$ turns out to be a matrix of the form $\prod_{\eta=1}^k\text{diag}(\mathbf{a}_\eta)\ \mathbf{S} \ \text{diag}(\mathbf{b})$, where $\mathbf{a}_{\eta,j}=\exp(-\frac{1}{\tau}c\theta_{\eta,v,j})$ if $u_{\eta,j}\neq 0$ else $\mathbf{a}_{\eta,j}=1$ (recall that $\forall u_{\eta,j}\neq 0: u_{\eta,j}=c$). Denoting $\text{diag}(\mathbf{a})=\prod_{\eta=1}^k\text{diag}(\mathbf{a}_\eta)$, $\mathbf{\Gamma}=\text{diag}(\mathbf{a})\ \mathbf{S} \ \text{diag}(\mathbf{b})$ is the unique matrix after normalizing $\mathbf{S}=\exp(\mathbf{W}/\tau)$ w.r.t.\ all row-wise and column-wise distributions, according to Sinkhorn's theorem \citep{sinkhorn1967concerning}. \cref{thm:extended-sk} and \cref{coro:extended-sk-L1} in the main paper derive its speed of convergence.

%   transforming $\mathbf{S}=\exp(\mathbf{W}/\tau)$ into $\mathbf{\Gamma}$ by alternative row-wise and column-wise normalization. In other words, Sinkhorn's theorem \citep{sinkhorn1967concerning} could be generalized to the multi-set case, and the convergence speed is already derived in \cref{thm:extended-sk} and \cref{coro:extended-sk-L1}. $\mathbf{\Gamma}$ is necessarily a matrix of the form $\text{diag}(\mathbf{a})\mathbf{S}\text{diag}(\mathbf{b})$ that belongs to all marginals.

%   Assuming that Sinkhorn's theorem \citep{sinkhorn1967concerning} is generalizable to the multi-set case (seems to be true as its convergence is already proved by), $\mathbf{\Gamma}$ is necessarily a matrix of the form $\text{diag}(\mathbf{a})\mathbf{S}\text{diag}(\mathbf{b})$ that belongs to all marginals. And we have $\mathbf{S}=\exp(\mathbf{W}/\tau)$.

\yanr{\subsection{The Algorithm with Entropic Regularizer}}
\begin{algorithm}[h!]
   \caption{Sinkhorn for Multi-Set Marginals with Entropic Regularizer}
\begin{algorithmic}[1]
   \STATE {\bfseries Input:} Score matrix $\mathbf{W}\in \mathbb{R}^{m\times n}$, entropic regularizer $\tau$, $k$ sets of marginals $\mathbf{V}\in\mathbb{R}_{\geq 0}^{k\times m},\mathbf{U}\in\mathbb{R}_{\geq 0}^{k\times n}$.
   \STATE Apply entropic regularizer $\mathbf{S}=\exp(\mathbf{W}/\tau)$;
   \STATE Initialize $\Gamma_{i,j}=\frac{s_{i,j}}{\sum_{i=1}^m s_{i,j}}$;
   \REPEAT
   \FOR{$\eta=1$ {\bfseries to} $k$}
   \STATE ${\Gamma}_{i,j}^{\prime} = \frac{{\Gamma}_{i,j}v_{\eta,i}}{\sum_{j=1}^n {\Gamma}_{i,j}u_{\eta,j}}$; $\triangleright$ normalize w.r.t.\ $\mathbf{v}_\eta$
   \STATE ${\Gamma}_{i,j} = \frac{{\Gamma}_{i,j}^{\prime}u_{\eta,j}}{\sum_{i=1}^m {\Gamma}_{i,j}^{\prime}u_{\eta,j}}$; $\triangleright$ normalize w.r.t.\ $\mathbf{u}_\eta$
   \ENDFOR
   \UNTIL{convergence}
   %\STATE {\bfseries Return:} normalized matrix $\mathbf{\Gamma}\in[0,1]^{m\times n}$
\end{algorithmic}
\end{algorithm}

\subsection{Theoretical Results with Entropic Regularizer}

If the entropic regularizer is involved, the converging rate of multi-set Sinkhorn w.r.t.\ $L_1$ error becomes:
% \begin{itemize}
%     \item For the single-set Sinkhorn, as derived from Theorem 3.1, its convergence time
%     \begin{equation}
%         t=\mathcal{O}\left(\frac{h^2(\log\Delta+\alpha^\prime/\tau)}{\epsilon^2}\right)
%     \end{equation}
    
%     \item For our multi-set Sinkhorn, as derived from Corollary 3.5, its convergence time
    \begin{equation}
        t=\mathcal{O}\left(\frac{\hat{h}^2k(\alpha^\prime/\tau+\log\Delta)}{\epsilon^2}\right),
    \end{equation}
% \end{itemize}
where $$\alpha^\prime= \max_{i,j}  w_{i,j} - \min_{i,j:w_{i,j}>-\infty}  w_{i,j}.$$

The other notations have the same definition as in the main paper. It shows that the number of iterations scales almost linearly with $1/\tau$. The derivation is straightforward since $\alpha$ is the only affected term after considering the entropic regularizer.

\subsection{Empirical Further Study of the Entropic Regularizer}

Our {\mname} can naturally handle continuous constraints. For continuous optimization problems (e.g. portfolio allocation) the output can be directly used as the feasible solution. When it comes to problems requiring discrete decision variables, we show in the following study that our {\mname} still owns the ability to encode the constraints by adjusting $\tau$.

For discrete optimization problems (such as TSP and GM), during training, we relax the discrete binary region $\{0, 1\}$ into continuous $[0,1]$ to make neural networks trainable with gradient-based methods. In \cref{sec:tsp,sec:gm}, the discrete solutions from continuous outputs are recovered via post-processing during inference. Post-processing is quite common in neural solvers for discrete optimization due to the continuous nature of neural networks. Let's take TSP as an example: RL-based methods \citep{kool2018attention,kwon2021matrix} need to manually design the decoding strategy to ensure every node is visited but only once; supervised learning methods \citep{joshi2019efficient} output a continuous heatmap over which search is performed.
 
As for our LinSAT, once the network is well-trained, we can set smaller $\tau$ to get outputs that are closer to a discrete feasible solution during inference. Here we conduct experiments on TSP and GM to evaluate the LinSAT's capability to maintain discrete feasible solutions. Using the same trained network, we adjust $\tau$ to get different continuous outputs $\widetilde{X}$. Then we directly round $\widetilde{X}$ to get discrete solution $X$: $X_{i,j} = 1$ if $\widetilde{X}_{i,j} \ge 0.5$; $X_{i,j} = 0$ if $\widetilde{X}_{i,j} < 0.5$. In \cref{tab:different_tau_tsp,tab:different_tau_gm}, we report the ratio that directly rounded solutions are feasible and the corresponding evaluation metrics for feasible solutions.

\begin{table}[tb!]
    \centering
    \caption{Feasible ratio and tour length comparison among different TSP solver configurations.}
    \begin{tabular}{r|rc|rc}
    \toprule
        & \multicolumn{2}{c|}{TSP-SE} & \multicolumn{2}{c}{TSP-PRI} \\
       Config  & Feasible Ratio & Tour Length & Feasible Ratio & Tour Length\\
    \midrule
        $\tau = 0.1$, Rounding & 1.19\% & 3.897 & 1.52\% & 3.997 \\
        $\tau = 0.05$, Rounding & 20.95\% & 3.904 & 23.26\% & 4.063 \\
        $\tau = 0.01$, Rounding & 86.69\% & 3.964 & 85.32\% & 4.102 \\
        $\tau = 0.005$, Rounding & 89.35\% & 3.969 & 83.63\% & 4.108 \\
        $\tau = 0.1$, Beam Search & 100.00\% & 3.811 & 100.00\% & 3.943 \\
    \bottomrule
    \end{tabular}
    \label{tab:different_tau_tsp}
\end{table}

\begin{table}[tb!]
    \centering
    \caption{Feasible ratio and average F1 among different partial-GM solver configurations.}
    \begin{tabular}{r|rc}
    \toprule
       Config  & Feasible Ratio & Mean F1 \\
    \midrule
        $\tau = 0.1$, Rounding & 56.78\% & 0.4756 \\
        $\tau = 0.05$, Rounding & 83,99\% & 0.5547 \\
        $\tau = 0.02$, Rounding & 95.21\% & 0.5748 \\
        $\tau = 0.01$, Rounding & 96.20\% & 0.5673 \\
        $\tau = 0.05$, Hungarian-Top-$\phi$ & 100.00\% & 0.6118 \\
    \bottomrule
    \end{tabular}
    \label{tab:different_tau_gm}
\end{table}

Results show that our {\mname} can recover most feasible solutions with a simple rounding strategy under a small temperature. However, smaller $\tau$ requires more iterations to converge, which makes inference slower: in TSP-SE, $\tau = 0.005$ requires 1,000 iterations to converge and the total inference time for $\tau = 0.005$, Rounding is 37s, which is longer than 19s of $\tau = 0.1$, Beam Search. And the quality of solutions is also not competitive with larger $\tau$ with post-processing. 
  
In summary, the aforementioned further study shows that:
1) Given smaller $\tau$, our {\mname} owns the ability to recover discrete feasible solutions directly.
2) Using a larger $\tau$ with post-processing steps (as done in our main paper) is a cost-efficient choice considering both efficiency and efficacy.

\yanr{
\section{The Feasibility Assumption Explained}
\label{sec:feasibility-assump}
For multi-set Sinkhorn, an assumption is made in Eq.~(\ref{eq:define-z}) that the marginal distributions must have a non-empty feasible region. We explain this assumption with an example derived from positive linear constraints.

As shown in \cref{sec:linsat}, every set of positive linear constraints could be equivalently viewed as a set of Sinkhorn marginals. If we transform the positive linear constraints to Sinkhorn's marginals:
\begin{itemize}
    \item If the positive linear constraints have a non-empty feasible region, e.g.
    \begin{equation}
        x_1 + x_2 \leq 1, x_3+x_4\leq 1, x_1 + x_3 \leq 1, x_2 + x_4 \leq 1,
    \end{equation}
    the corresponding marginals are 
    \begin{subequations}
    \begin{align}
        &\mathbf{u}_1=[1, 1, 0, 0, 1], \mathbf{v}_1^\top=[1, 2] \quad \triangleright\text{for $x_1 + x_2 \leq 1$}\\
        &\mathbf{u}_2=[0, 0, 1, 1, 1], \mathbf{v}_2^\top=[1, 2] \quad \triangleright\text{for $x_3 + x_4 \leq 1$}\\
        &\mathbf{u}_3=[1, 0, 1, 0, 1], \mathbf{v}_3^\top=[1, 2] \quad \triangleright\text{for $x_1 + x_3 \leq 1$}\\
        &\mathbf{u}_4=[0, 1, 0, 1, 1], \mathbf{v}_4^\top=[1, 2] \quad \triangleright\text{for $x_2 + x_4 \leq 1$}\\
    \end{align}
    \end{subequations}
    A feasible solution $x_1=1, x_2=0, x_3=0, x_4=1$ corresponds to a valid transportation plan,
    \begin{equation}
        \mathbf{\Gamma}=\left[\begin{array}{ccccc}
             1 & 0 & 0 & 1 & 0  \\
             0 & 1 & 1 & 0 & 1
        \end{array}\right].
    \end{equation}
    Thus Eq.~(\ref{eq:define-z}) is feasible in such cases.
    
    \item If the positive linear constraints have no feasible region, e.g.\ 
    \begin{equation}
        x_1 + x_2 \geq 2, x_3+x_4\geq 2, x_1 + x_3 \leq 1, x_2 + x_4 \leq 1,
    \end{equation}
    Eq.~(\ref{eq:define-z}) is infeasible. Such infeasible marginals are not expected, and {\mname} will not converge.
\end{itemize}
As a side note, the last column and the second row of $\mathbf{\Gamma}$ are set as separate dummy columns/rows for different marginals in our {\mname}, and their converged values may be different for different marginals. Empirically, {\mname}'s converging property is not affected after adding the dummy columns/rows.
}

% \section{Details of Experiments} \label{sec:exp_detail}
\section{Details of Case Study I: Neural Solver for Traveling Salesman Problem with Extra Constraints} \label{sec:detailed_tsp_exp}
\subsection{Hyper-parameters and Implementation} \label{sec:tsp_hyper}
In both TSP-SE and TSP-PRI, we use a 3-layer Transformer to encode the 2-D coordinates into hidden vectors. Then the hidden vectors are projected into the pre-projected matrix using a 3-layer MLP with ReLU activation. Dimensions of hidden states for both the Transformer and the MLP are set to 256, and the head number of multi-head attention in the Transformer is set to 8.

We train the model for 100 epochs for both TSP-SE and TSP-PRI. In each epoch, we randomly generate 256,000 instances as the training set of this epoch. The batch size is set to 1,024. Adam optimizer is used for training and the learning rate is set to 1e-4.

During inference, we use beam search to post-process the continuous matrix $\widetilde{\mathbf{X}}$ output by the network to get the binary matrix $\mathbf{X}$. The width of the beam for beam search is set to 2,048.

Our model runs on a single NVIDIA GeForce RTX 2080Ti GPU with 11GB memory.
\subsection{Baseline Methods}
\paragraph{MIP} MIP methods directly use Gurobi to solve the formulation in Sec.\ref{sec:TSP_fml}, e.g. an integer programming problem with linear constraints and quadratic objective. The time limit per instance is set to 2s/20s.

\paragraph{Nearest Neighbor} Nearest Neighbor is a greedy heuristic for TSP. For TSP-SE, in each iteration, the nearest node (except the ending node) to the starting node is selected as the next node to visit. Then the selected node becomes the new starting node in the next iteration. After all nodes except the ending node are visited, the tour directly connects to the ending node. 

For TSP-PRI, in the $m$-th iteration, if the priority node has not been visited, the priority node will be selected as the next node to satisfy the priority constraint. 

\paragraph{Insertion Heuristic} Insertion Heuristic first uses the starting and ending nodes to construct a partial tour. In each iteration, a new node is selected and inserted to the partial tour to extend it. For a selected node, it is inserted in the position where the tour length increase is minimized. Formally, we use $T = \{\pi_1, \pi_2, \dots, \pi_m\}$ to denote a partial tour with $m (m < n)$ nodes. Assuming the selected new node is $u^*$, then it is inserted behind the $i^*$-th node in the partial tour:
\begin{equation}
    i^* = \argmin_{1 \le i \le m - 1} D_{{\pi_i},u^*} + D_{u^*,{\pi_{i+1}}} - D_{{\pi_i},{\pi_{i+1}}}
\end{equation}
According to the different new node selection processes, there are different variants of insertion heuristic:

\textit{Nearest Insertion} selects the nearest node to the partial tour:
\begin{equation}
    u^* = \argmin_{u \not \in T}\min_{v \in T}D_{u,v}
\end{equation}
\textit{Farthest Insertion} selects the farthest node from any node in the partial tour:
\begin{equation}
    u^* = \argmax_{u \not \in T}\min_{v \in T}D_{u,v}
\end{equation}
\textit{Random Insertion} randomly selects the new node to insert.

To satisfy the priority constraint of TSP-PRI, if the priority node is already the $(m + 1)$-th node in the partial tour, we can not insert new nodes in front of it.

\paragraph{Standard Solver} TSP-SE can be converted to standard TSP by adding a dummy node. The distance from the dummy node to starting and ending nodes is 0 and the distance to other nodes is infinity. Then we can use start-of-the-art methods for standard TSP, i.e. \textit{Gurobi(MTZ)/Concrode/LKH3}, to solve TSP-SE. 

However, converting TSP-PRI to standard TSP is non-trivial, making it hard to use standard solvers to solve TSP-PRI.

\paragraph{Attention Model} \textit{Attention Model}~\citep{kool2018attention} is an RL-based autoregressive model for standard TSP. We modify its decoding process so that it can solve TSP-SE and TSP-PRI. Because TSP-SE ends in the ending node instead of constructing a circle, we use the ending node embedding to replace the first node embedding in the original paper during decoding. For TSP-PRI, if the priority node has not been visited within the first $m-1$ steps, it will be visited in the $m$-th step. The training process and hyper-parameters are the same as our model in Sec.~\ref{sec:tsp_hyper}.

\textit{MIP, Insertion Heuristic, Standard Solver} run on a Intel(R) Core(TM) i7-7820X CPU; \textit{Nearest Neighbor, Attention Model} run on a NVIDIA GeForce RTX 2080Ti GPU with 11GB memory.

\begin{table}[tb]
    \centering
    \caption{Ablation study on normalizing randomly generated matrices on TSP.}
    \begin{tabular}{r|ll}
        \toprule
        Method & TSP-SE & TSP-PRI \\ 
        \midrule
        Random Pre-Projected Matrix & 7.414 (19s) & 7.426 (18s) \\
        Trainable Pre-Projected Matrix & 5.546 (26m51s) & 5.646 (27m37s) \\
        Transformer Feature Matrix & \textbf{3.811} (19s) & \textbf{3.943} (18s) \\
        \bottomrule
    \end{tabular}
    \label{tab:norm_random_TSP}
\end{table}

\subsection{Further Ablation Study}

In the following study, we show that {\mname} can normalize matrices outside a neural network. The study involves two variants of our TSP solver presented in \cref{sec:tsp}:
\textbf{1) Random Pre-Projected Matrix}: Apply our {\mname} to a randomly generated matrix and do beam search over it; \textbf{2) Trainable Pre-Projected Matrix}: Randomly initialize the pre-projected matrix, view the matrix as trainable parameters and use the same training process as in the main paper to optimize it. \textbf{Transformer Feature Matrix} is our TSP solver in \cref{sec:tsp}.

The average tour length and total inference time are shown in \cref{tab:norm_random_TSP}. The Random Pre-Projected Matrix cannot provide useful guidance to beam search, thus its performance is poor. Trainable Pre-Projected Matrix performs better, but it is easy to stick at local optima without global features extracted by the neural network. Moreover, updating the pre-projected matrix requires multiple forward and backward passes of {\mname}, making this method much more time-consuming.

This ablation study proves the feasibility of {\mname} working outside a neural network, while it also shows the necessity of neural networks to get high-quality solutions in an efficient time.

\subsection{Visualizations}
Fig.~\ref{fig:tsp_se_case} and Fig.~\ref{fig:tsp_pri_case} show some route plan cases of TSP-SE and TSP-PRI using different methods. Our {\mname} is able to get near-optimal solutions in a short time, especially for TSP-PRI where getting an optimal solution is very time-consuming.

\begin{figure*}[tb!]
    \centering
    \includegraphics[width=0.24\textwidth, height = 0.24\textwidth]{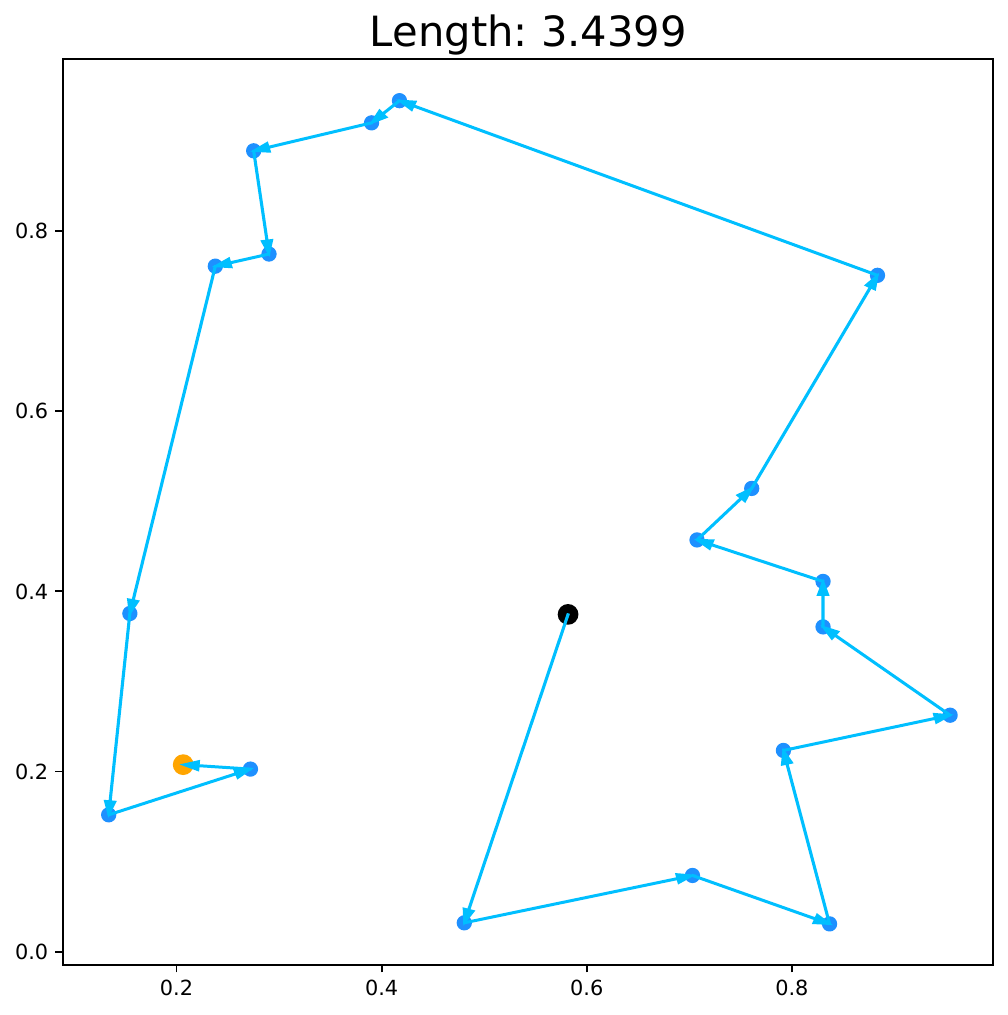}
    \includegraphics[width=0.24\textwidth, height = 0.24\textwidth]{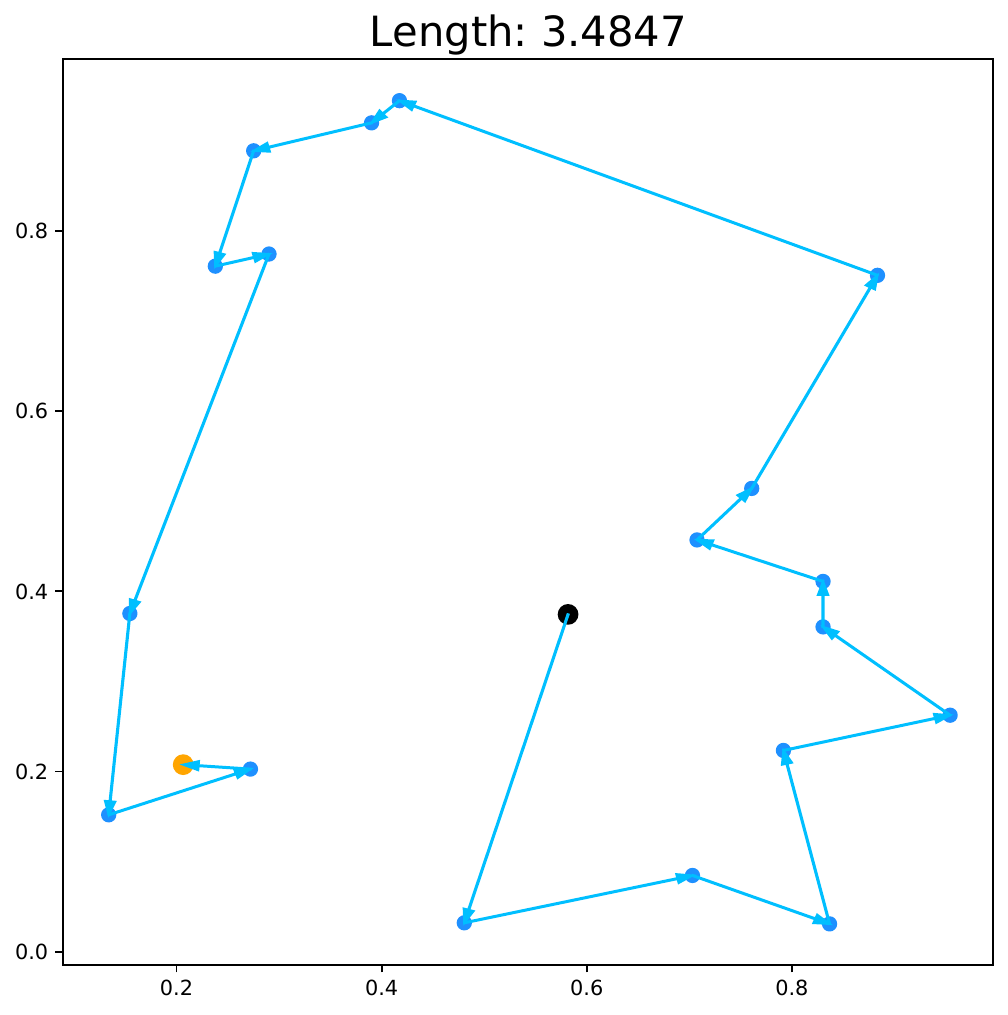}
    \includegraphics[width=0.24\textwidth, height = 0.24\textwidth]{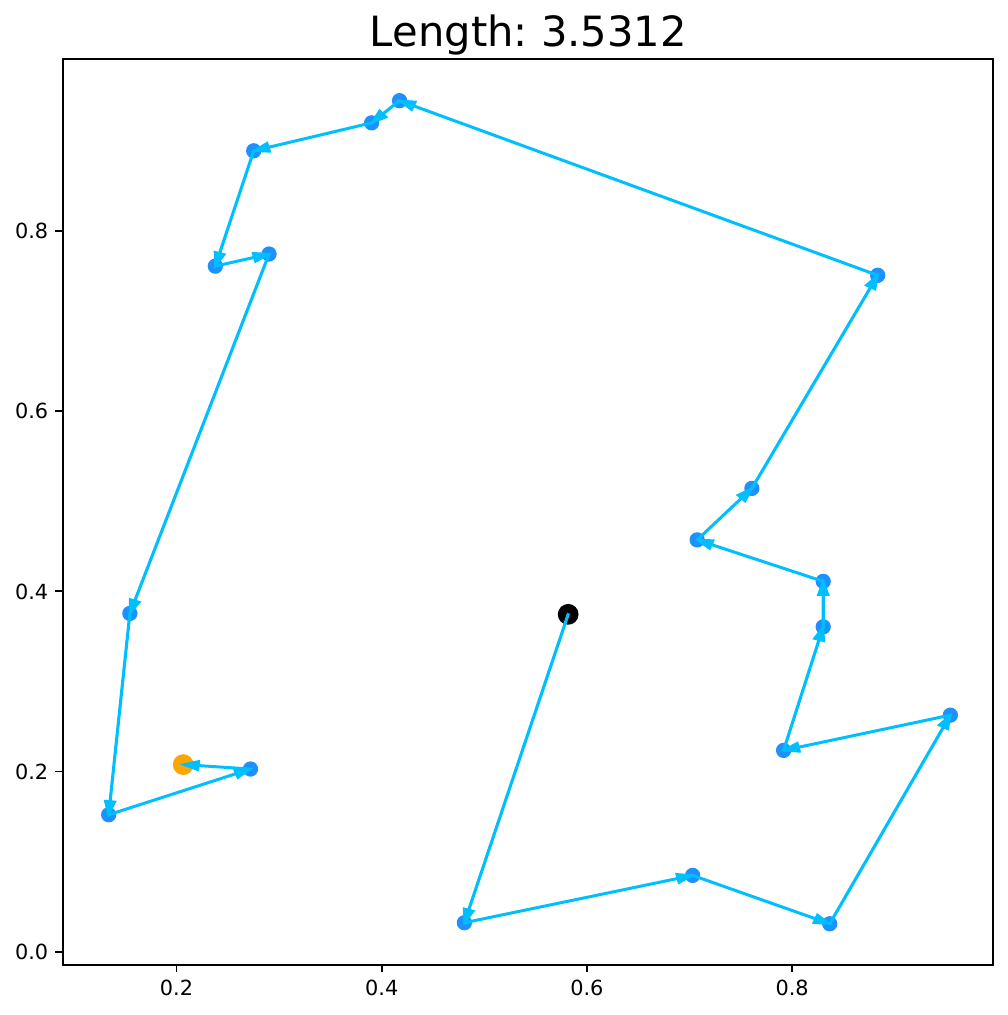}
    \includegraphics[width=0.24\textwidth, height = 0.24\textwidth]{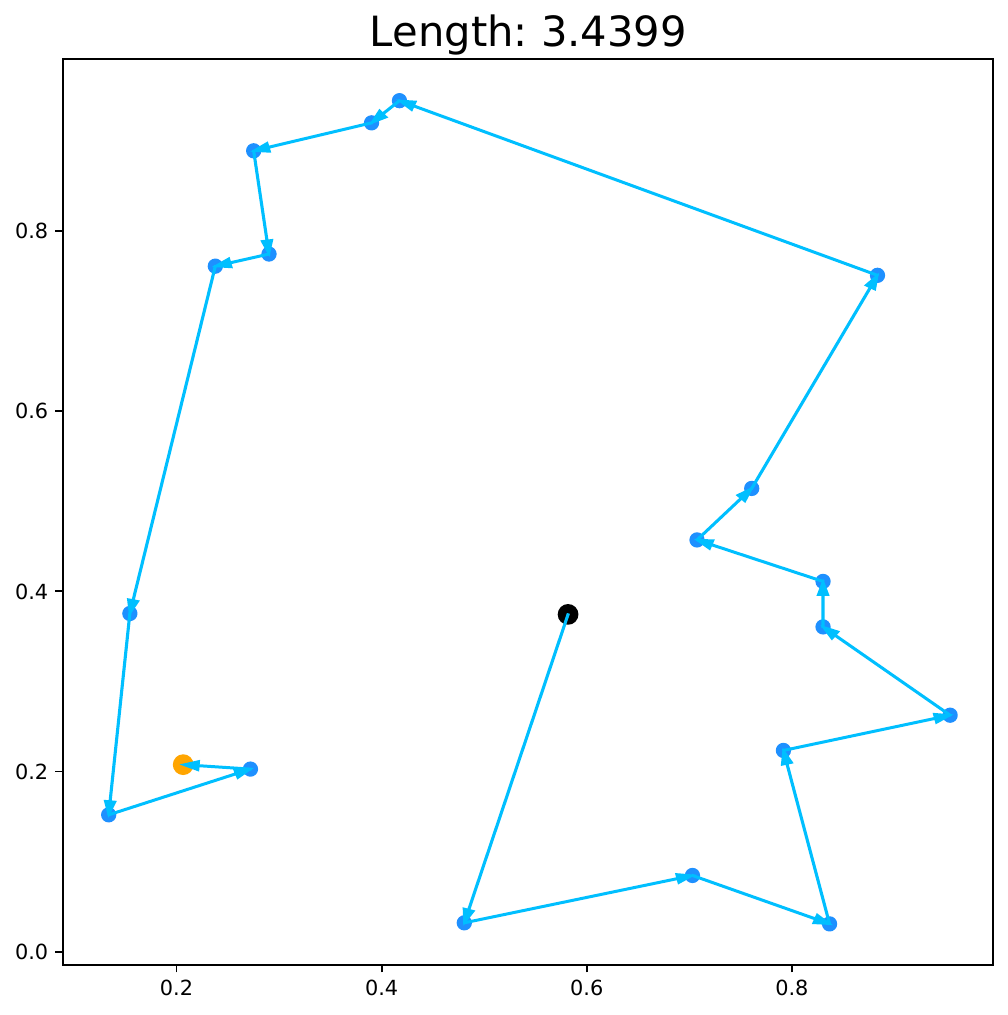}
    \\
    \subfigure[Optimal]{\includegraphics[width=0.24\textwidth, height = 0.24\textwidth]{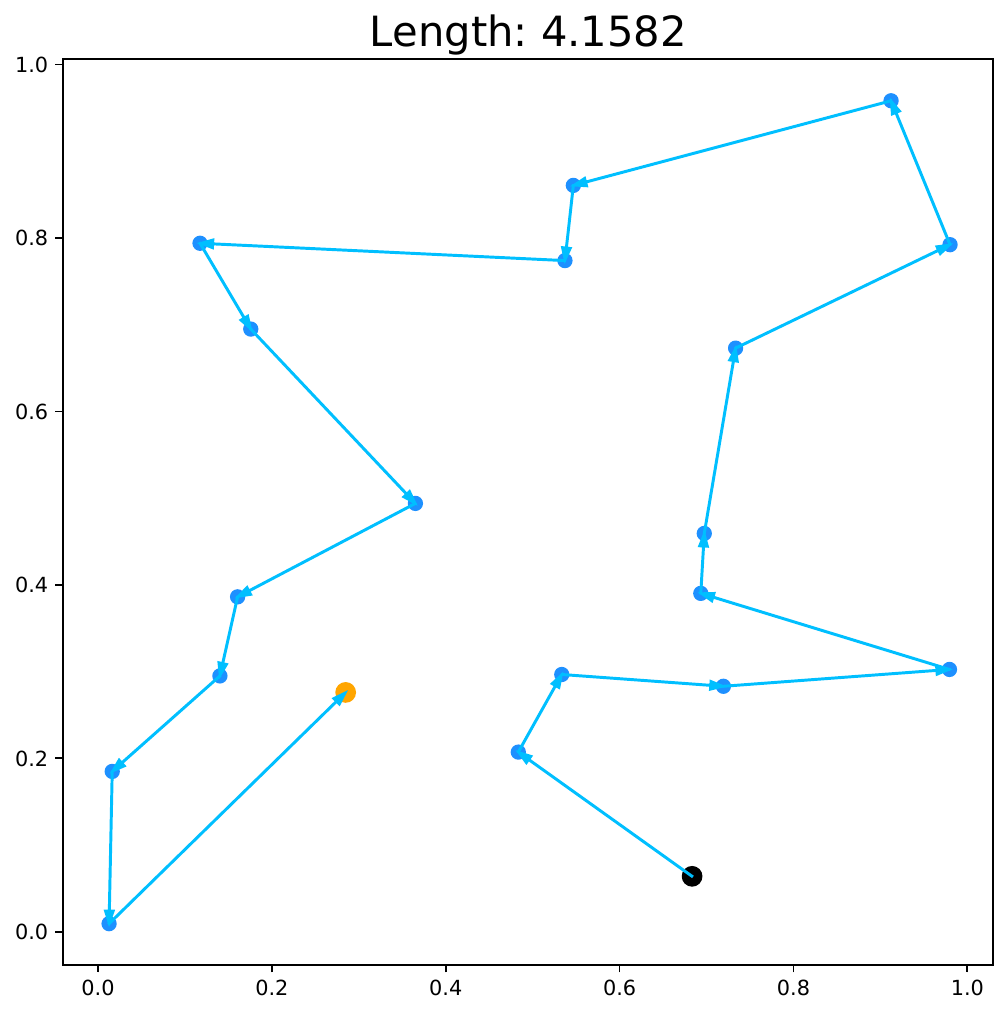}}
    \subfigure[Gurobi (Sec.\ref{sec:TSP_fml}, 10s)]
    {\includegraphics[width=0.24\textwidth, height = 0.24\textwidth]{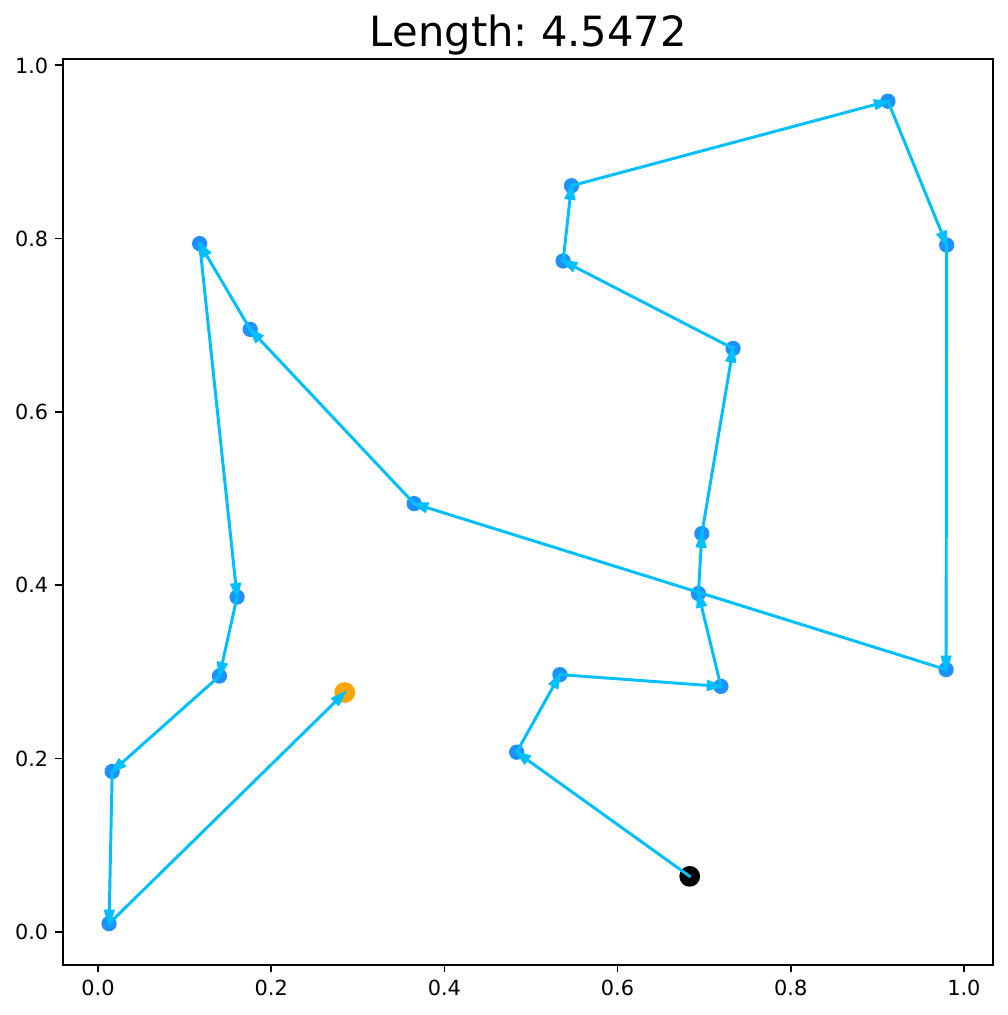}}
    \subfigure[Attention Model]{\includegraphics[width=0.24\textwidth, height = 0.24\textwidth]{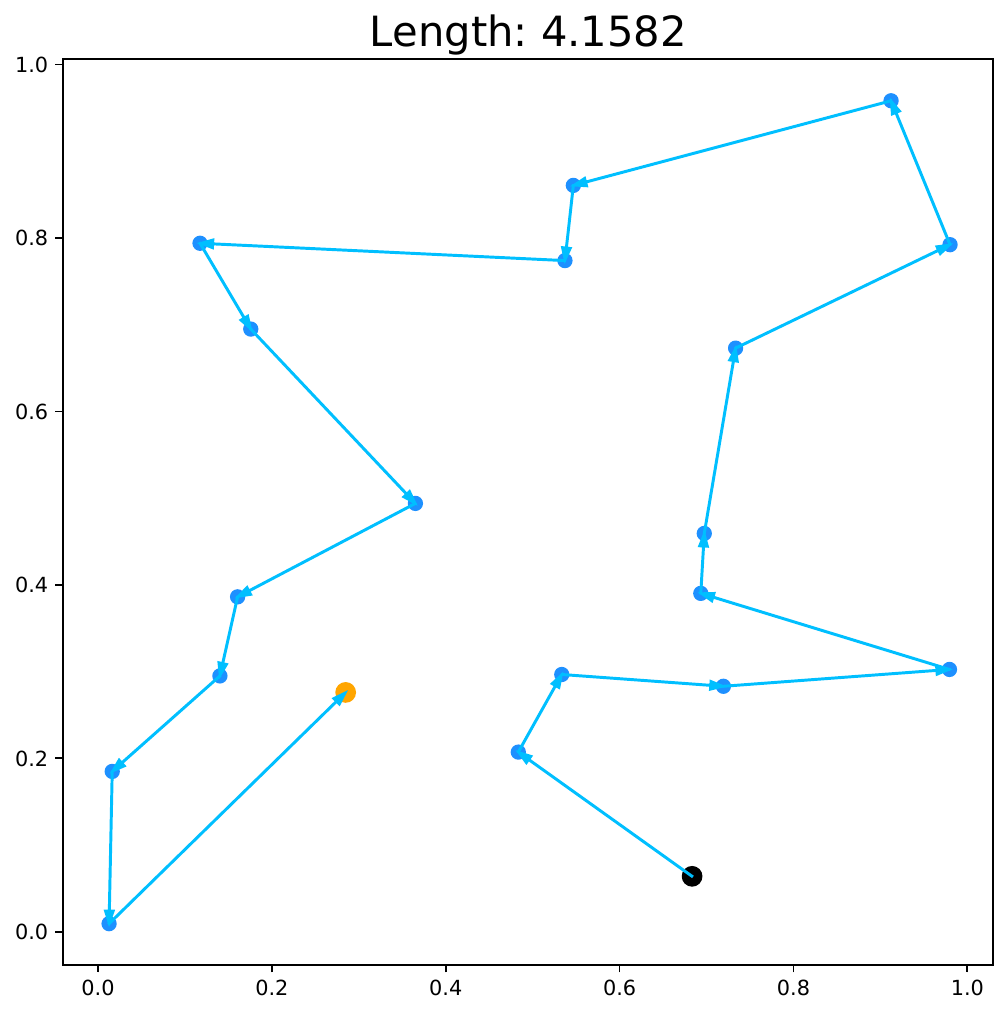}}
    \subfigure[\mname]{\includegraphics[width=0.24\textwidth, height = 0.24\textwidth]{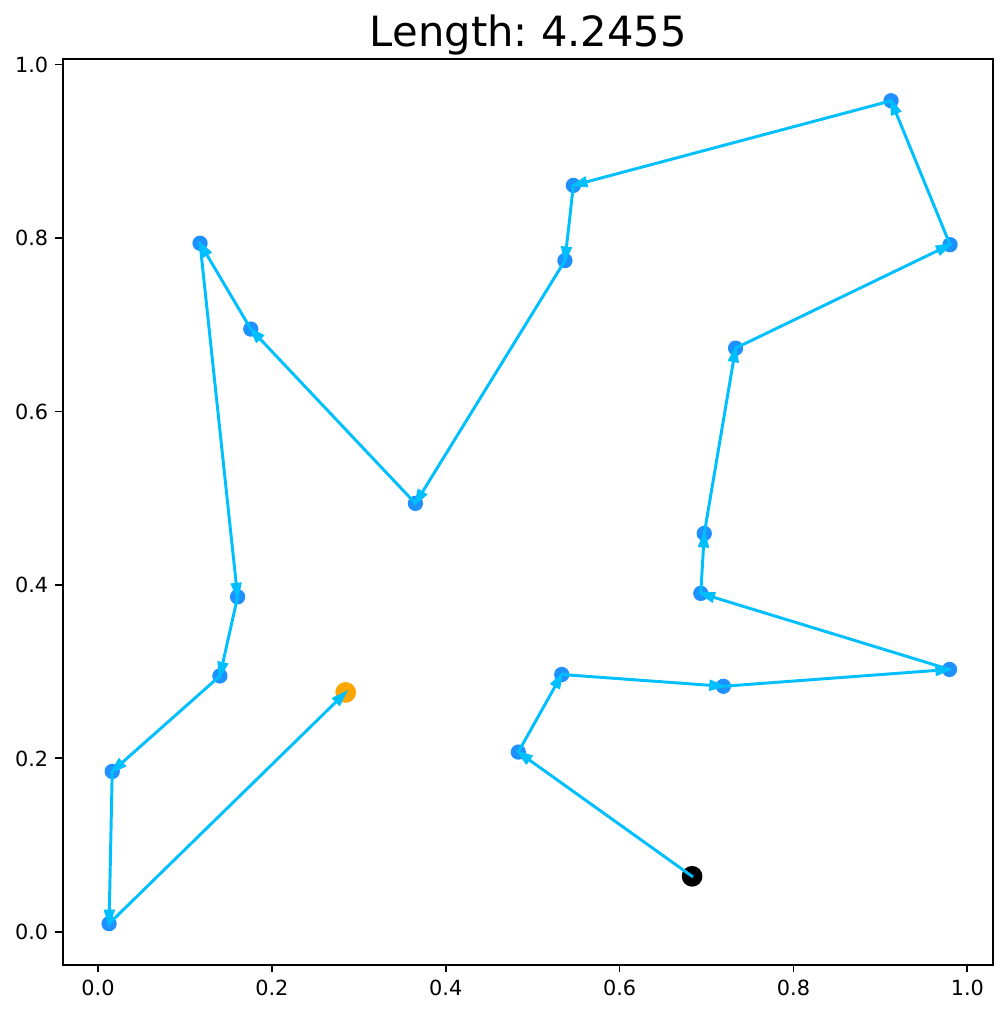}}
    \\
    \caption{Route plan cases of TSP-SE. Each row represents one instance and each column represents one method. Black points stand for starting cities and orange points stand for ending cities. The optimal tours are solved by Gruobi (MTZ).}
    \label{fig:tsp_se_case}
\end{figure*}

\begin{figure*}[tb!]
    \centering
    \includegraphics[width=0.24\textwidth, height = 0.24\textwidth]{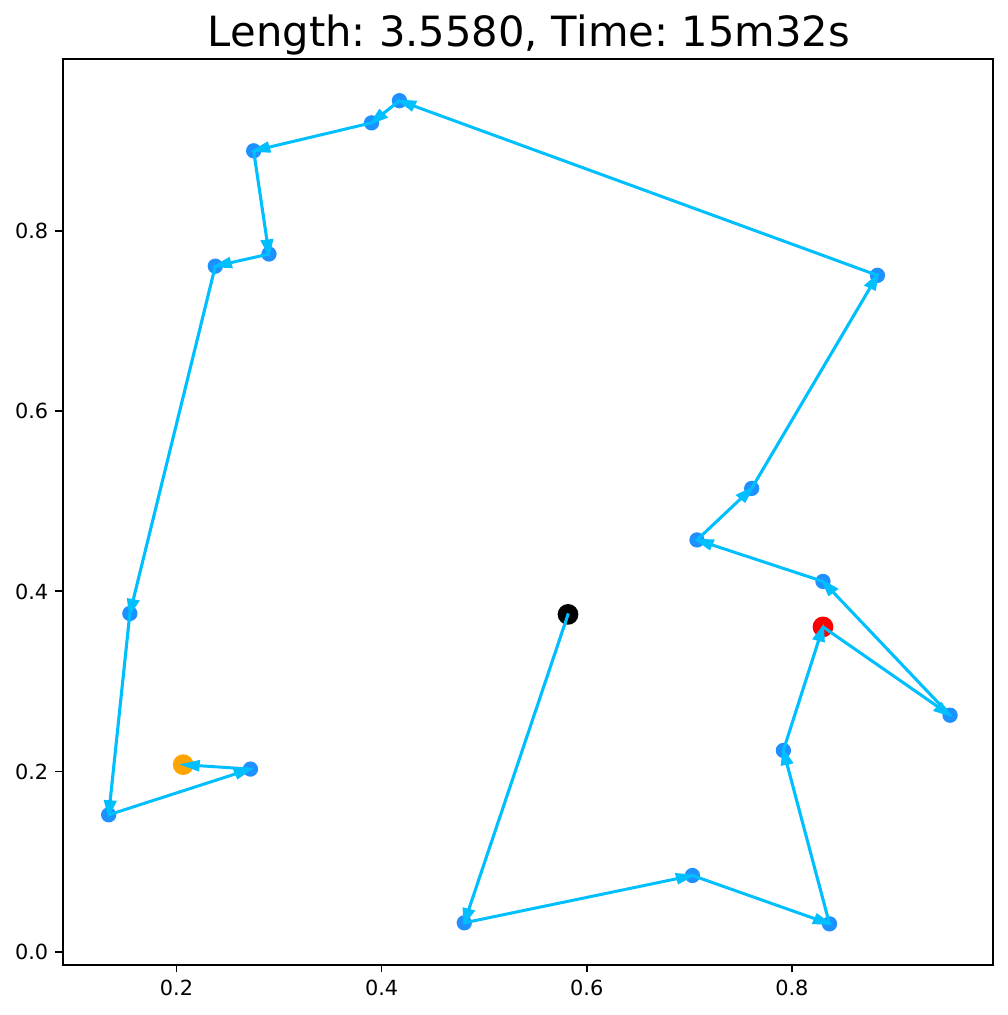}
    \includegraphics[width=0.24\textwidth, height = 0.24\textwidth]{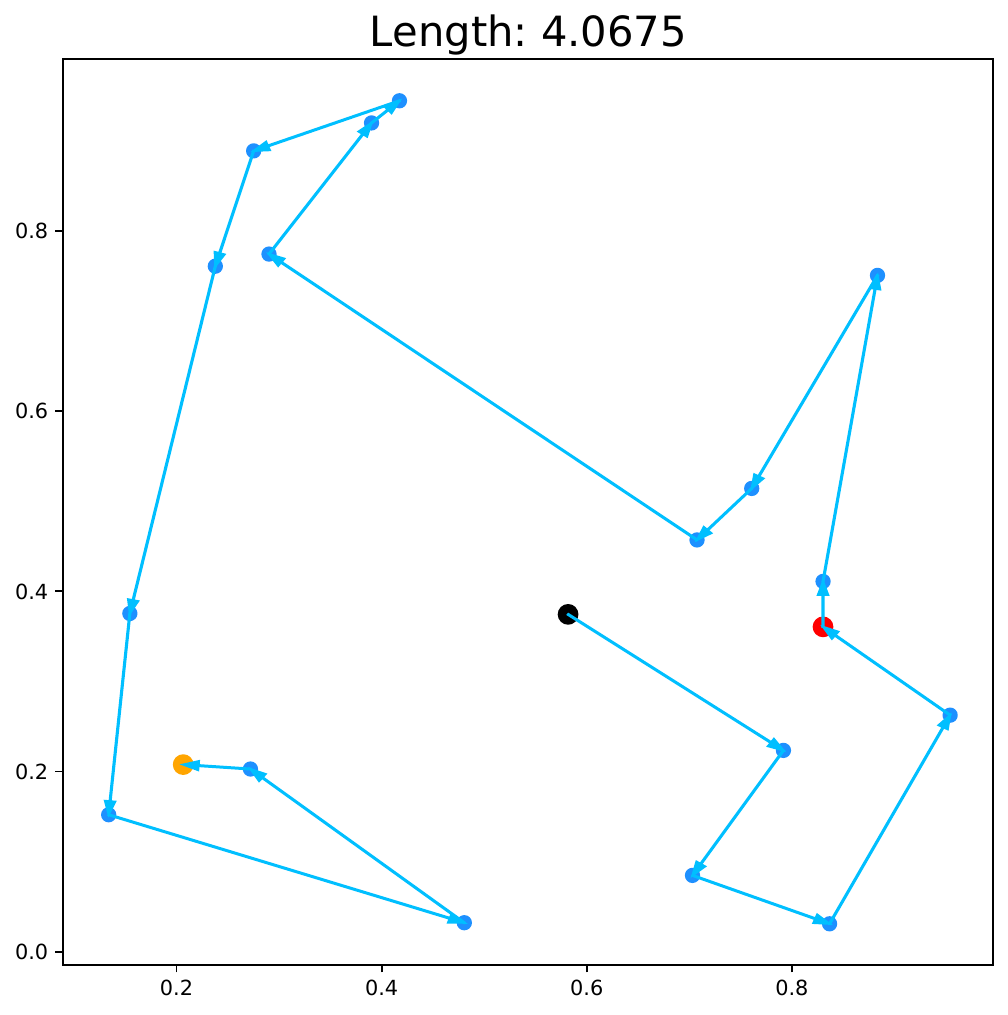}
    \includegraphics[width=0.24\textwidth, height = 0.24\textwidth]{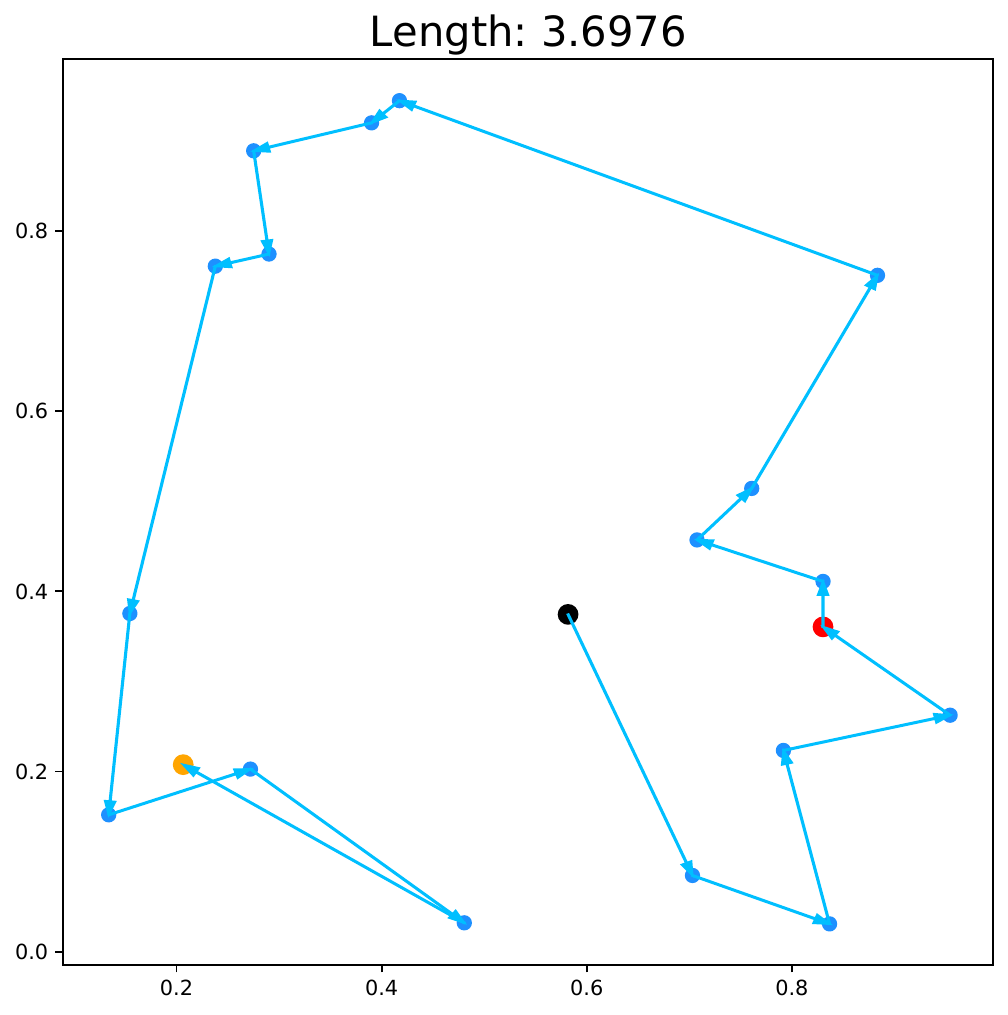}
    \includegraphics[width=0.24\textwidth, height = 0.24\textwidth]{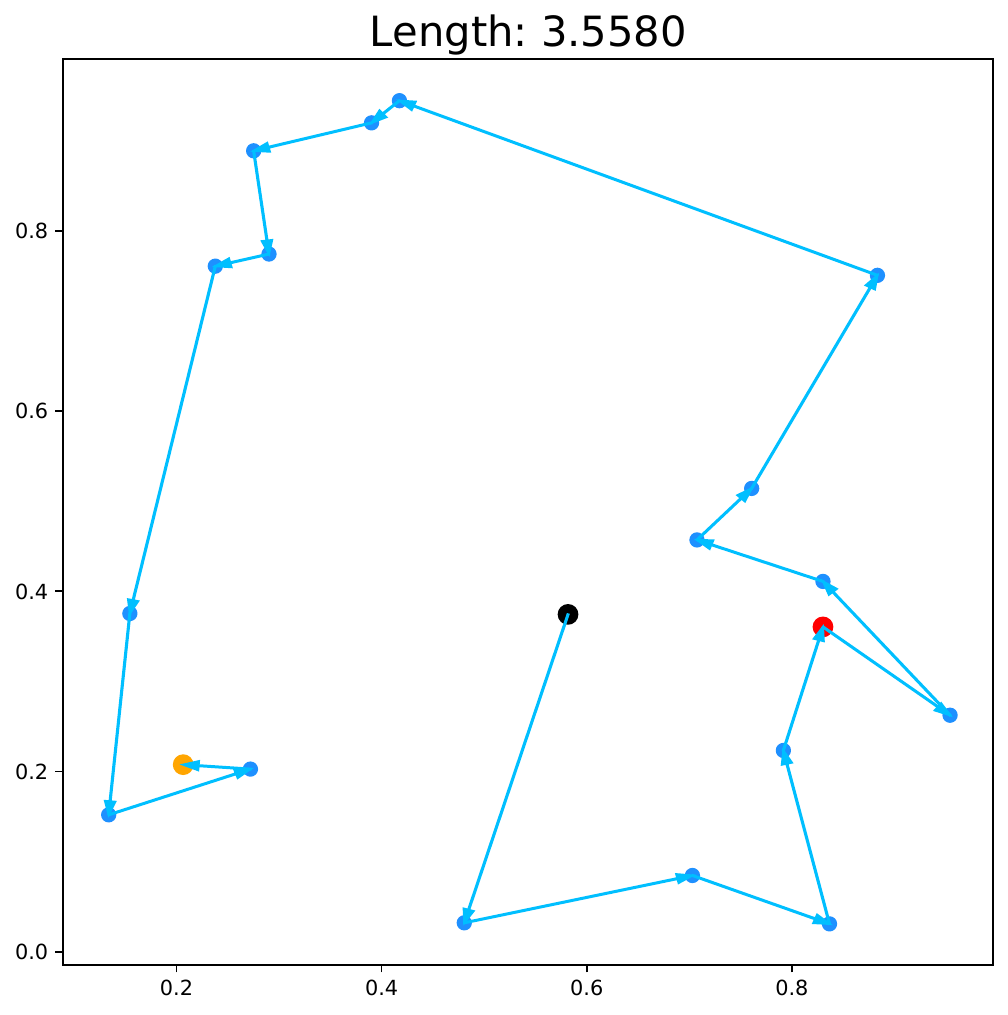}
    \\
    \subfigure[Optimal]{\includegraphics[width=0.24\textwidth, height = 0.24\textwidth]{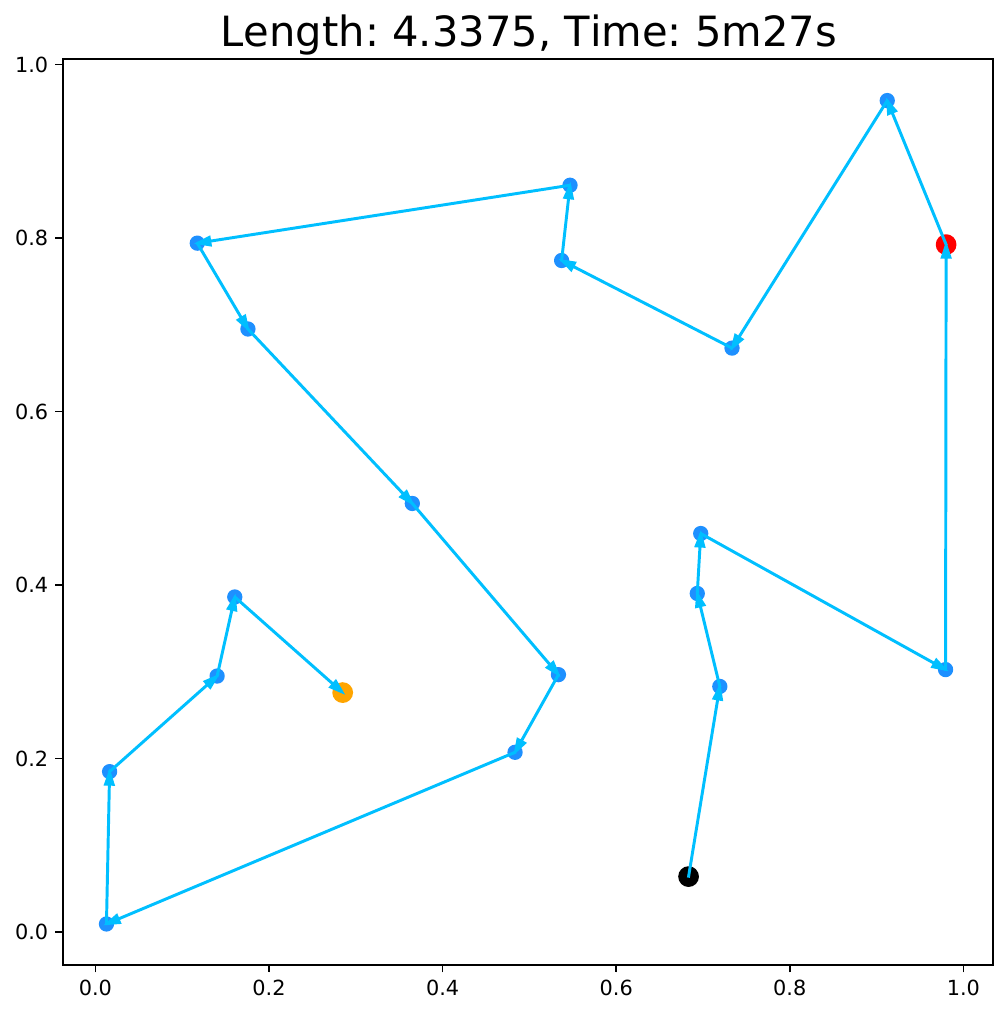}}
    \subfigure[Gurobi (Sec.\ref{sec:TSP_fml}, 10s)]{\includegraphics[width=0.24\textwidth, height = 0.24\textwidth]{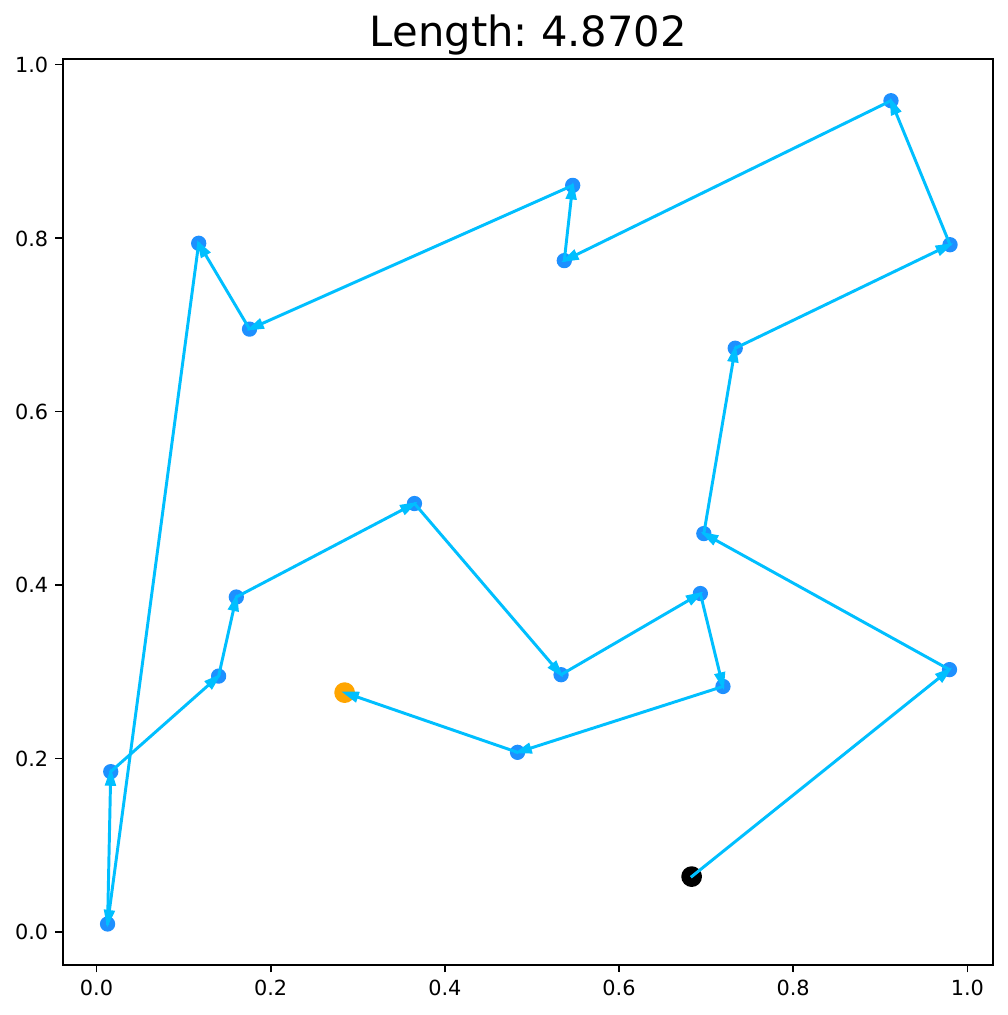}}
    \subfigure[Attention Model]{\includegraphics[width=0.24\textwidth, height = 0.24\textwidth]{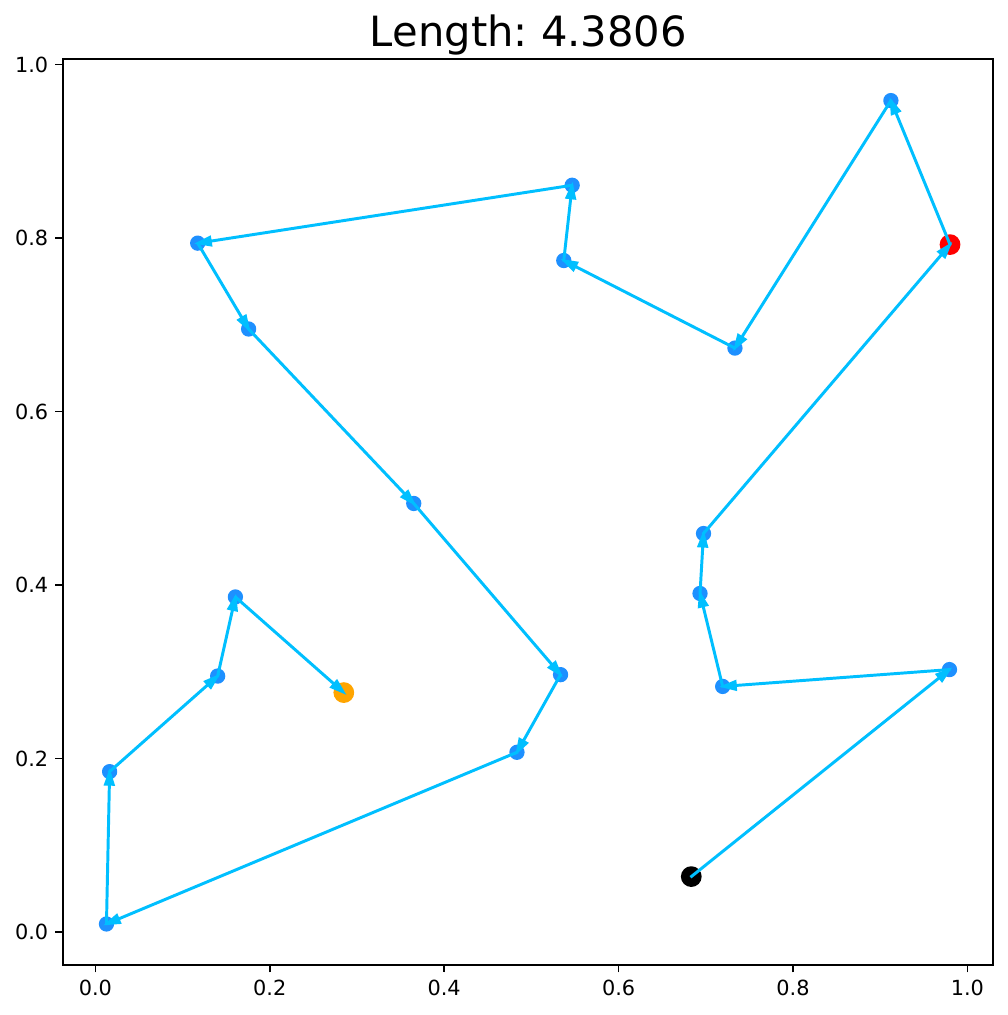}}
    \subfigure[\mname]{\includegraphics[width=0.24\textwidth, height = 0.24\textwidth]{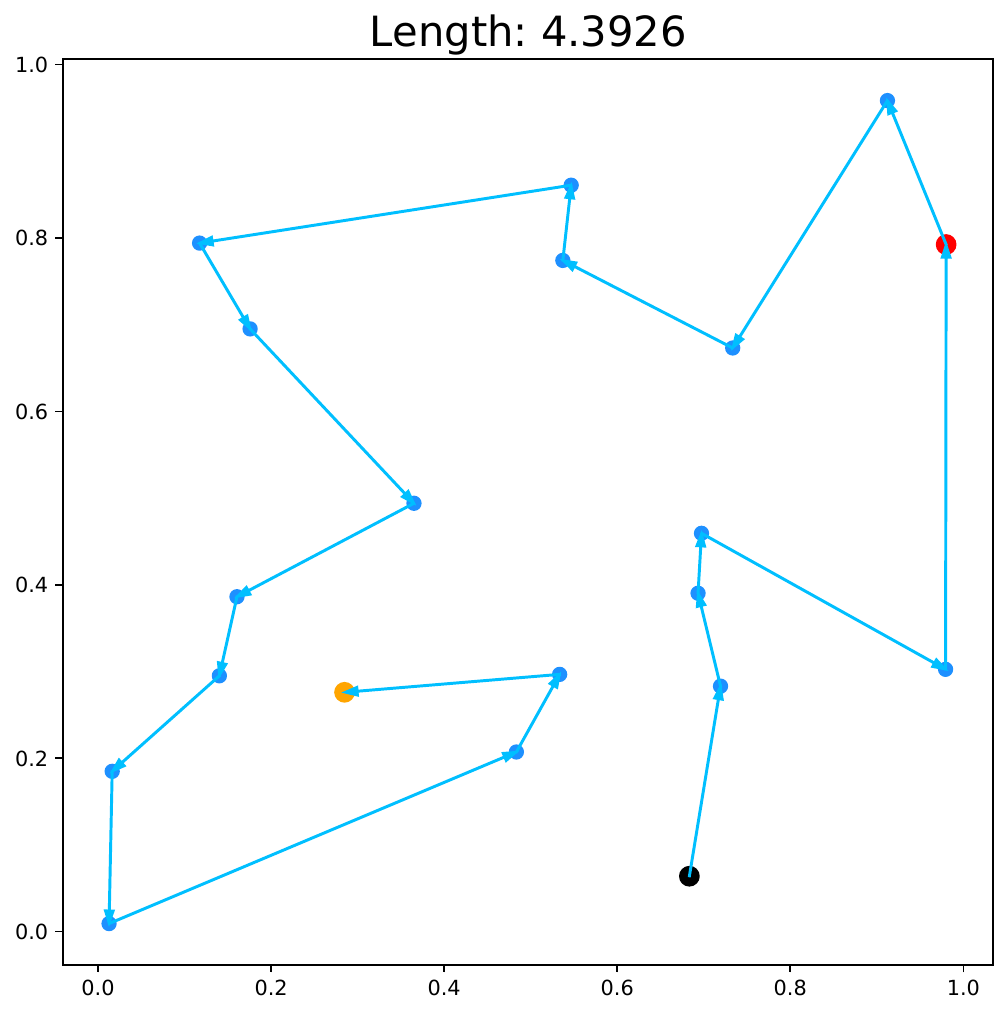}}
    \\
    \caption{Route plan cases of TSP-PRI.  Each row represents one instance and each column represents one method. Black points stand for starting cities, orange points stand for ending cities and red points stand for priority cities that need to be visited within the first 5 steps. The optimal tours are solved by Gurobi (Sec.\ref{sec:TSP_fml}) without a time limit, and the running time to get the corresponding optimal solutions are listed in the titles.}
    \label{fig:tsp_pri_case}
\end{figure*}

\section{Experiment Testbed}
All our experiments are run on our workstation with Intel(R) Core(TM) i7-7820X CPU, NVIDIA GeForce RTX 2080Ti GPU, and 11GB memory.

\begin{table}[tb!]
    \centering
    \caption{Traveling salesman problem with extra constraints' time cost. Similar to other neural solvers for standard TSP~\citep{kool2018attention,kwon2021matrix} the post-processing step (e.g.\ beam search, Monte Carlo tree search) is the most time-consuming, yet it is necessary to achieve better results.}
    \begin{tabular}{r|ccc}
    \toprule
        Task & Neural Network & {\mname} & Beam Search \\
    \midrule
        TSP-SE & 0.3\% & 12.4\% & 87.3\% \\
        TSP-PRI & 0.3\% & 12.6\% & 87.1\% \\
    \bottomrule
    \end{tabular}
    \label{tab:tsp_time}
\end{table}

\begin{table}[tb!]
    \centering
    \caption{Partial graph matching's time cost. Hungarian Top-$\phi$ is our discretization step discussed at the end of \cref{sec:gm_network}.}
    \begin{tabular}{r|ccc}
    \toprule
        Module & Neural Network & {\mname} &  Hungarian Top-$\phi$ \\
    \midrule
        Proportion of Time &  17.1\% & 81.4\% & 1.5\% \\
    \bottomrule
    \end{tabular}
    \label{tab:gm_time}
\end{table}

\begin{table}[tb!]
    \centering
    \caption{Portfolio allocation's time cost. Note that since the decision variables of portfolio allocation are continuous, it does not need any post-processing/discretization steps.}
    \begin{tabular}{r|cc}
    \toprule
        Module & Neural Network & {\mname} \\
    \midrule
        Proportion of Time & 20.1\% & 79.9\% \\
    \bottomrule
    \end{tabular}
    \label{tab:port_time}
\end{table}

\yanr{
\section{Study of Time Costs}
In our case studies, {\mname} has higher time costs than neural networks, which is in our expectation because the cost of one Sinkhorn iteration could be roughly viewed as one layer of neural network from the unfolding perspective. The complexity of neural networks in our case studies is relatively low, since too many layers may cause the over-smoothing issue, and designing new networks is beyond the scope of this paper. We summarize the proportion of inference time in all 3 case studies in \cref{tab:tsp_time,tab:gm_time,tab:port_time}.

Besides, we compare the timing statistics of {\mname} with another regularized projection method -- the differentiable CVXPY layers \citep{AgrawalNIPS19}. Specifically, we conduct a case study of transforming a random matrix into a doubly-stochastic matrix. Both methods achieve doubly-stochastic matrices, and the timing statistics are in \cref{tab:cvxpy_time}. LinSAT is more efficient in this case study. Extra speed-up may be achieved when the input scales up and switching {\mname} to GPU (CVXPY is CPU-only).

\begin{table}[tb!]
    \centering
    \caption{Timing statistics of projecting a random matrix in to doubly-stochastic (on CPU, in seconds).}
    \begin{tabular}{c|ccc}
      \toprule
       Method  & Forward & Backward & Total \\
       \midrule
       {\mname}  & 0.0382 & 0.0250 & 0.0632 \\
       CVXPY   & 0.1344 & 0.0042 & 0.1386 \\
       \bottomrule
    \end{tabular}
    \label{tab:cvxpy_time}
\end{table}
}

% \section{Limitations}
% We would like to discuss the following potential limitations of this paper, where the corresponding future work are discussed in the conclusion of the main paper:

% 1) If handling the same constraint (e.g.\ the one-on-one matching constraint in bijective graph matching), the efficiency of {\mname} is inferior to the classic single-set Sinkhorn algorithm, because {\mname} needs to take more iterations among $k$ separate constraints. A possible direction to improve the efficiency is to merge non-overlapped constraints together and jointly handles multiple constraints in a single iteration. Since the iteration itself is GPU-friendly, a reduced number of iterations would very likely improve the efficiency of {\mname}.

% 2) Though our {\mname} incorporates many previously studied satisfiability forms, it only handles the positive linear constraints, and there is still room to the most general (both positive and negative) linear constraints. 

%%%%%%%%%%%%%%%%%%%%%%%%%%%%%%%%%%%%%%%%%%%%%%%%%%%%%%%%%%%%%%%%%%%%%%%%%%%%%%%
%%%%%%%%%%%%%%%%%%%%%%%%%%%%%%%%%%%%%%%%%%%%%%%%%%%%%%%%%%%%%%%%%%%%%%%%%%%%%%%

\end{document}